\documentclass[journal]{IEEEtran}
\usepackage{amsmath,amsfonts}
\usepackage{algorithmic}
\usepackage{algorithm}
\usepackage{array}
\usepackage[caption=false,font=scriptsize,labelfont=sf,textfont=sf]{subfig}
\usepackage{textcomp}
\usepackage{stfloats}
\usepackage{url}
\usepackage{verbatim}
\usepackage{graphicx}
\usepackage{cite}
\usepackage{amsthm}
\newtheorem{theorem}{Theorem}
\newtheorem{lemma}{Lemma}
\newtheorem{proposition}{Proposition}
\newtheorem*{remark}{Remark}

\usepackage{enumerate}
\usepackage{multirow}
\usepackage{color} % highlight for reviewer
\newcommand{\highlight}[1]{\color{black}{#1}} %red,yellow,green,blue,cyan
\newcommand{\unhighlight}[0]{\color{black}{}}

\begin{document}
\title{High-order regularization dealing with ill-conditioned robot localization problems}
% \author{Xinghua Liu,~\IEEEmembership{Student Member,~IEEE,} Ming Cao,~\IEEEmembership{Fellow,~IEEE}
\author{Xinghua Liu, Ming Cao
        % <-this % stops a space
% \thanks{Manuscript received 8 November 2024; revised 25 February 2025; revised 18 March 2025; accepted 25 March 2025. Date of publication xx xx 2025; date of current version xx xx 2025.}

\thanks{The work was supported in part by the Netherlands Organization for Scientific Research (NWO-Vici-19902), and the China Scholarship Council.}
\thanks{The authors are with Engineering and Technology Institute (ENTEG), the University of Groningen, 9747 AG Groningen, the Netherlands (e-mail: \textit{\{xinghua.liu,  m.cao\}}@rug.nl).}% <-this % stops a space
% \thanks{Color versions of one or more figures in this article are available at https://xxx}

% \thanks{Digital Object Identifier xxx}
\thanks{© 2025 IEEE.  Personal use of this material is permitted.  Permission from IEEE must be obtained for all other uses, in any current or future media, including reprinting/republishing this material for advertising or promotional purposes, creating new collective works, for resale or redistribution to servers or lists, or reuse of any copyrighted component of this work in other works.}

}

\markboth{Accepted by IEEE Transactions on Robotics}%
{Shell \MakeLowercase{\textit{et al.}}: A Sample Article Using IEEEtran.cls for IEEE Journals}


\maketitle

\begin{abstract}
In this work, we propose a high-order regularization method to solve the ill-conditioned problems in robot localization. Numerical solutions to robot localization problems are often unstable when the problems are ill-conditioned. A typical way to solve ill-conditioned problems is regularization, and a classical regularization method is the Tikhonov regularization. It is shown that the Tikhonov regularization is a low-order case of our method. We find that the proposed method is superior to the Tikhonov regularization in approximating some ill-conditioned inverse problems, such as some basic robot localization problems. The proposed method overcomes the over-smoothing problem in the Tikhonov regularization as it uses more than one term in the approximation of the matrix inverse, and an explanation for the over-smoothing of the Tikhonov regularization is given. Moreover, one \emph{a priori} criterion which improves the numerical stability of the ill-conditioned problem is proposed to obtain an optimal regularization matrix. As most of the regularization solutions are biased, we also provide two bias-correction techniques for the proposed high-order regularization. The simulation and experimental results using an Ultra-Wideband sensor network in a 3D environment are discussed, demonstrating the performance of the proposed method.
\end{abstract}


\section{Introduction}
\IEEEPARstart{T}{}he growing usage of mobile robots, e.g., indoor parking lots, cargo automation management in automatic factories, autonomous vehicles, and fire rescue \cite{couturier_review_2021}, gives rise to the high demand of localization techniques, especially under limited or constrained sensing capabilities. A challenging problem in robot localization is to deal with the \emph{ill-conditioned} situations\cite{bertoni2022perceiving}, in which small changes in the measurements lead to huge variations in the location solutions \cite{li2017three}. Regularization introduces penalty terms on variables into original problems when solving localization problems and is one of the most direct approaches to preventing the appearance of ill conditions. The most typical regularization method is Tikhonov regularization (TR) \cite{benning2018modern} (sometimes referred to, under generalization, as the $L_{2}$-norm regularization \cite{turk2018identification}). An improved TR is proposed by Fuhry \emph{et al.} \cite{fuhry2012new}, which shows the relation between the truncated singular value decomposition (TSVD) method and the TR method.

The improved TR provides a special choice of the regularization matrix, which is further developed by Yang \emph{et al.} \cite{yang2015modified}, Noschese \emph{et al.}\cite{noschese2016some}, Mohammady \emph{et al.}\cite{mohammady2020extension} and Cui \emph{et al.}\cite{cui2020special}. A fractional filter method \cite{klann2008regularization} is proposed by Klann and developed by Mekoth who presents the fractional TR method \cite{mekoth2021fractional}. A thorough overview of modern regularization methods is covered in \cite{benning2018modern}. For a regularization method, such as the TR method, one of the key procedures is the determination of regularization parameters. There are many methods to select the regularization parameter, such as L-curve and generalized cross-validation (GCV) methods \cite{hansen1993use}. However, these parameter-selecting methods are posterior or empirical methods, which means that to acquire the regularization parameters, iterative algorithms or data related to the specific problems and environments are required. Therefore, they are computationally expensive \cite{gazzola2021iteratively}. The bias of the regularization solution is also a main topic in regularization. Bias-correction methods using iteratively estimated variance components, proposed by Xu \emph{et al.}\cite{xu2006variance}, Shen \emph{et al.} \cite{shen2012bias} and Ji \emph{et al.}\cite{ji2022adaptive}, are proven to be effective to calculate the bias and achieve unbias solutions for some problems\cite{shen2012bias}. 

Applying and developing regularization to robot localization problems and their related areas in robotics have been an intensively studied area over the past decade \cite{nerurkar2009distributed, wang2014toa, yuan2019toa, liouane2021regularized, ning2020method, pan2022multi-station}. Nerurkar \emph{et al.} apply a Tikhonov regularization-based algorithm, the Levenberg-Marquardt (LM) algorithm, to estimate robots' poses in cooperative robot localization problems\cite{nerurkar2009distributed}. Wang \emph{et al.} \cite{wang2014toa} solve an ill-conditioned linearized localization problem using the TR method. Yuan \emph{et al.} \cite{yuan2019toa} present a factor graph-based framework and a regularized least square (RLS) for passive localization problems. Ning \emph{et al.} \cite{ning2020method} study a localization problem with truncated singular value decomposition regularization. Liouane \emph{et al.} \cite{liouane2021regularized} improve the location estimation accuracy of the sensor in a wireless sensor network with a regularized least square method. Pan \emph{et al.} \cite{pan2022multi-station} work on applying TR to the multi-sensor localization problem. Tang \emph{et al.} \cite{tang2018localization} and Hu \emph{et al.} \cite{hu2019multiple} study underwater localization problems with regularization to deal with ill-conditioned situations. Details of the Tikhonov regularization-based methods are studied  \cite{qian2021vehicle,yu2021passive,fang2019robust} to solve specific problems in related areas of robot localization. The Tikhonov regularization is introduced into a kernel learning method to provide the generalization ability of a kernel model and improve the localization behavior of the kernel model when the robot is in high dynamics conditions\cite{qian2021vehicle}. Yu \emph{et al.}\cite{yu2021passive} study the three-dimensional passive localization problem that is vulnerable to the influence of measurement error and apply TR to solve the ill-conditioned localization problem. Fang \emph{et al.} \cite{fang2019robust} propose a robust Newton algorithm based on TR by avoiding the ill-conditioned Hessian matrix caused by the bad initial value. Their results show improvements in the accuracy, stability, and convergence performance of the proposed algorithms. Other applications of Tikhonov regularization include path-planning for small-scale robots \cite{zhao20203d}, motion control of a magnetic microrobot \cite{xing2022optimized}, and calibration of a robot positioning processes \cite{li2022diversified}. While Han \emph{et al.} work on applying $L_1$ regularization for collision avoidance in safe robot navigation \cite{han2023rda}, Wu \emph{et al.} focus on the analysis of the optimal TR matrix and its application for localization problems with global navigation satellite system (GNSS) \cite{wu2022regularized}. However, calculating the proposed optimal regularization matrix requires prior information (the first and second central moments) of an unknown integer ambiguity vector. Some regularization methods, e.g., Bayesian regularization, which can be simplified as Tikhonov regularization in some cases \cite{polson2019bayesian, feng2021element}, are used to train deep neural networks (DNNs) using the LM algorithm and then the DNNs are used to localize indoor robots to achieve high positioning precision\cite{zhang2019high}. Regularization in safe and smooth robot localization when considering continuous movements of robots\cite{dümbgen2023safe} is studied by Dümbgen. Dong \emph{et al.} \cite{dong2023trajectory} study a localization problem of a flying robot with a single UWB ranging measurement and IMU measurement, in which they introduce velocity penalty terms to the objective function of their problem and solve the problem by the LM method with regularization. The use of regularization in vision tasks, such as feature detection and matching, 3-D point cloud registration, and visual homing, which are applied in simultaneous localization and mapping (SLAM) for robot localization problems \cite{strasdat2012visual,ma2018nonrigid, guizilini2022learning}, is also studied as there are many ill-conditioned situations in SLAM. However, additional efforts need to be made to analyze the impact and explain the effect of regularization on the solutions.

This work focuses on ill-conditioned problems with a tractable inverse mapping. This paper proposes a high-order regularization method to solve ill-conditioned problems and get stable numerical solutions to robot localization tasks. The main contributions of this work are summarized as follows.
\begin{enumerate}
    \item A new high-order regularization method is constructed. The proposed high-order regularization method explains and overcomes the over-smoothing problem in the Tikhonov regularization. It also bridges the least square method and the Tikhonov regularization method mathematically by approximating the inverse of a matrix. According to the approximation expression of the matrix inverse, the lower bound and the upper bound of the approximation residual are provided, which is used to derive the error bounds of the high-order regularization solutions, as well as the error bounds of the Tikhonov regularization solutions as TR is proven to be a low-order case of high-order regularization.
    \item We provide a new restructuring method for simplifying the selection of the regularization matrix. The presented simplification method is used to improve the numerical stability of the algorithm by modifying the condition number using the regularization matrix.
    \item We construct one \emph{a priori} criterion to obtain the optimal regularization matrix for the proposed high-order regularization method. The regularization matrix is explicitly and directly calculated according to the proposed criterion. The closed-form solution of the optimal regularization matrix under the criterion is also provided for some cases.
    \item Two bias-correction methods are provided to drive a possible unbiased high-order regularization.
    \item A bias-correction algorithm with a sliding window for the robot localization problems is proposed, and the ill-conditioned situation for the robot localization problem is discussed.
    \item The implementation of the high-order regularization method and the bias-correction algorithm for robot localization in 3D environments are also presented.
\end{enumerate}

The rest of this paper is organized as follows. In Section \ref{introduce_tikhonov_regularization}, regularization for ill-conditioned problems in robot localization is introduced. In Section \ref{propose_high-order_regularization}, a high-order regularization method is proposed, which shows the relationship between the LS method and the TR method and the reason for the over-smoothness of the TR solution. The simplification method and a prior criterion are proposed to obtain the optimal regularization matrix of the high-order regularization method. Two bias-correction techniques for the high-order regularization method are also presented in Section \ref{propose_high-order_regularization}. In Section \ref{robot_localization_sliding_window}, the application and the ill-conditioned issue in robot localization using a sensor network and in other robotics problems are discussed, and a bias-correction method with a sliding window is described. In Section \ref{simulation_experiment}, some simulation results, as well as experimental results, are given to demonstrate the performance of the proposed method. The summarization and future work of this paper are provided in Section \ref{conclusions}.

\subsection*{Notation}
Some notations are summarized as follows

\begin{tabular}{ll}
$\mathbb{R}, \mathbb{N}$  &\begin{tabular}{l} 
The set of real numbers, natural numbers; and\\
$\mathbb{R}^n$, the $n$-dimension real vector space, \\
$\mathbb{R}^{m \times n}$, the set of the $m$-by-$n$ matrices. \\
\end{tabular} \\
$\mathrm{S^n}$ &\begin{tabular}{l} 
The set of $n$-dimension symmetry matrices; and\\
$\mathrm{S}_{+}^n$, the set of positive semidefinite matrices, \\
$\mathrm{S}_{++}^n$, the set of positive definite matrices.
\end{tabular} \\
$A$ &\begin{tabular}{l}  A matrix; and $A^{-1}$, the inverse of $A$, \\ 
$A^T$, the transpose of $A$,\\
$A^k$, the $k$th power of $A$ with $k \in \mathbb{N}$.\\
\end{tabular} \\
$x$ &\begin{tabular}{l} A variable; and $\hat x$, the estimation of $x$, \\
$x^*$, the true value of $x$.\\
\end{tabular} \\
$b$ &\, A vector; and $\|b\|$, the 2-norm of $b$.
\end{tabular}

\section{Ill-conditioned problems in robot localization and their Tikhonov regularization}\label{introduce_tikhonov_regularization}

To solve ill-conditioned robot localization problems, we propose a high-order regularization method and a theoretical solution for selecting the regularization parameter based on an \emph{a priori} criterion. The proposed method is superior to the TR method and overcomes the over-smoothing problem in the TR method.

A range of localization problems, such as robot localization problems using a sensor network, Ultra-Wideband (UWB) aided robot localization, pose initialization in SLAM \cite{sayed2005network,cao2006sensor,wang2017ultra,qin2018vins,guo2019ultra,zhang2016cooperative}, can be described using linear models; even when the original models are nonlinear, their linearized approximation \cite{song2019uwb} still leads to meaningful solutions by taking the solution of the linearized equation as the infinitesimal change. These nonlinear problems are solved by methods like the Gauss-Newton method, which mainly involves solving a linearized incremental equation. Some problems are described using specialized linear models, such as the weighted least squares form. However, these specialized forms can be transformed into the standard linear model. Consider a set of linear equations
\begin{equation} \label{problem_1}
Ax = b,
\end{equation}
or its related minimization problem
\begin{equation}\label{problem_min_1}
\min _{x \in \mathbb{R}^n}\|A x-b\|^2,
\end{equation}
where $x \in \mathbb{R}^n, A \in \mathbb{R}^{m \times n}, b \in \mathbb{R}^m$, and the norm $\|\cdot\|$ is the 2-norm. The least square (LS) method and its variations are often used in solving such linear problems described by equation \eqref{problem_1}. The solution of the LS method is given by
\begin{equation} \label{Ls_soluntion}
\begin{aligned}
& \hat{x}_{ls}=\left(A^T A\right)^{-1} A^T b.
\end{aligned}
\end{equation}
However, it is more challenging when the matrix $A$ is ill-conditioned (\cite{horn2012matrix}, Chapter 5.8); in this case, an approximate solution given by the TR method \cite{jia2020regularization} is handy. A standard form of the TR solution is found in \cite{jia2020regularization}, which is given by
\begin{equation}\label{TR_soluntion}
\begin{aligned}
& \hat{x}_{tr}=\left(A^T A+\mu^2 I\right)^{-1} A^T b,
\end{aligned}
\end{equation}
where $\mu$ is a given regularization parameter. This solution solves the following problem
\begin{equation} \label{TR_ap_problem_min}
\begin{aligned}
& \min _{x \in \mathbb{R}^n}\left\{\|A x-b\|^2+\mu^2\|x\|^2\right\}, 
\end{aligned}
\end{equation}
which is an approximate problem of the original problem \eqref{problem_min_1} with an extra penalty term called the regularization term.

It is known, though, that the TR solution is overly smooth \cite{klann2008regularization} without a proper justification of adding a regularization term to the original problem \eqref{problem_min_1}. We explicitly show that the TR solution is an approximate solution of the LS solution in terms of matrix inverse approximation by adding the regularization term to the original problem. Computing the inverse of the ill-conditioned matrix is the main reason that the LS solution is unstable.

\section{New high-order regularization method}\label{propose_high-order_regularization}
To improve the numerical stability of ill-conditioned problems, which, in other words, turns the ill-conditioned problems into well-conditioned ones, and get a numerically stable solution to the original problems \eqref{problem_1}, we propose a novel high-order regularization (HR) method. Then to obtain an optimal regularization matrix, we introduce an \emph{a priori} criterion, taking into account the trade-off between the estimation bias of the high-order regularization method and the improvement in the numerical stability of the approximation solution of the problem \eqref{problem_min_1}.

In the following subsections, we present the details of the proposed high-order regularization methods.

\subsection{Matrix inverse and the high-order regularization method} 
To obtain the high-order regularization solution, we first rewrite the normal matrix $A^T A$ as
\begin{equation}
A^T A=\left(I-R\left(A^T A+R\right)^{-1}\right)\left(A^T A+R\right), \end{equation}
then recall the matrix series \cite{horn2012matrix}
\begin{equation}
\begin{aligned}
& \left(I-R\left(A^T A+R\right)^{-1}\right)^{-1}=\sum_{i=0}^{\infty}\left(R\left(A^T A+R\right)^{-1}\right)^i, \\
&\text { if } \rho\left(R\left(A^T A+R\right)^{-1}\right)<1, \\
% &\text { if }\left\{R \mid \rho\left(R\left(A^T A+R\right)^{-1}\right)<1\right\} \\
\end{aligned}
\end{equation}
where $\rho(\cdot)$ is the spectral radius of a matrix. Thus, the inverse of the matrix $A^T A$ is
\begin{equation} \label{Approximate_expression}
\begin{aligned}
& \left(A^T A\right)^{-1}=\left(\left(I-R\left(A^T A+R\right)^{-1}\right)\left(A^T A+R\right)\right)^{-1} \\
& =\left(A^T A+R\right)^{-1}\left(I-R\left(A^T A+R\right)^{-1}\right)^{-1} \\
& =\left(A^T A+R\right)^{-1} \sum_{i=0}^{\infty}\left(R\left(A^T A+R\right)^{-1}\right)^i.
\end{aligned}
\end{equation}

Keeping only the first $k+1$ terms of the matrix power series in equation \eqref{Approximate_expression}, one has the $k$th-order regularization method \eqref{HR_soluntion}. Assuming that $A$ is known exactly, the proposed high-order regularization solution is
\begin{equation} \label{HR_soluntion}
\begin{aligned}
& \hat{x}^k_{hr}=\left(A^T A+R\right)^{-1} \sum_{i=0}^k\left(R\left(A^T A+R\right)^{-1}\right)^i A^T b,
\end{aligned}
\end{equation}
where the index $ k $ is the order of the proposed method, the regularization matrix $R$ satisfies the spectral radius condition $\rho\left(R\left(A^T A+R\right)^{-1}\right)<1$, and the matrix $A^T A$ and the new matrix $\left(A^T A+R\right)$ are nonsingular. We show that if $R \in \mathrm{S}_{+}^n$, the convex set of symmetric positive semidefinite (PSD) matrices, then the spectral radius condition is always satisfied. Then the nonsingular matrix $\left(A^T A+R\right) \in \mathrm{S}_{++}^n$, which is the convex set of the symmetric positive definite (PD) matrix. The equation \eqref{HR_soluntion} gives a $k$th-order regularized approximate solution to the equation \eqref{problem_1}.
The high-order regularization method, therefore, solves the approximate problem rather than the ill-conditioned problem \eqref{problem_1} directly (see \textbf{Appendix A}).

\begin{lemma}
The TR solution given in \eqref{TR_soluntion} is a special case of the proposed HR solution \eqref{HR_soluntion} for $k=0$.
\end{lemma}

\begin{proof}
For $k=0$, the HR solution reduces to
\begin{equation}\label{HR_solution_k0}
\hat{x}^0_{hr}=\left(A^T A+R\right)^{-1} A^T b,
\end{equation}
which is equivalent to the TR solution with $R=\mu^2 I$, and the spectral radius condition $\rho\left(R\left(A^T A+R\right)^{-1}\right)<1$ required in the matrix series is always satisfied.
\end{proof}

In general, select $k>0$, and a regularization matrix that meets the spectral radius condition to alleviate the over-smoothing of the TR method. For $k=1$, the HR solution is simplified to
\begin{equation}\label{HR_soluntion_k1}
\begin{aligned}
\hat{x}^1_{hr}=&\left(A^T A+R\right)^{-1} A^T b\\
&+\left(A^T A+R\right)^{-1} R\left(A^T A+R\right)^{-1} A^T b \\
=&\left ( I+\left (A^T A+R\right)^{-1} R \right)\left(A^T A+R\right)^{-1} A^T b,
\end{aligned}
\end{equation}
which is a better solution than the TR solution of problem \eqref{problem_min_1} in terms of the matrix inverse approximation of the matrix
$A^T A$. To alleviate the over-smoothness of the TR method, the equation \eqref{HR_soluntion_k1} shows that adding an extra term $\left(A^T A+R\right)^{-1} R\left(A^T A+R\right)^{-1} A^T b$ is equivalent to multiplying by an appropriate factor $\left ( I+\left (A^T A+R\right)^{-1} R \right)$.

\begin{lemma} \label{Lemma_2}
If $R$ is symmetric positive semidefinite and $A^T A$ is positive definite, then the spectral radius condition $\rho\left(R\left(A^T A+R\right)^{-1}\right)<1$ has always been satisfied.
\end{lemma}

\begin{proof}
If $R$ is symmetric positive semidefinite and $A^T A$ is symmetric positive definite, then $\left(A^T A+R\right)$ is symmetric positive definite and $A^T A+R \succ R$, where $\succ$ defines the Loewner partial ordering of symmetric positive definite matrices \cite{horn2012matrix}. Then according to \cite{horn2012matrix} (Theorem 7.7.3, p. 494), it follows that the spectral radius condition is satisfied.
\end{proof}

Since one can always design a symmetric regularization matrix $R$ for robot localization problems, according to Lemma \ref{Lemma_2}, the required spectral radius condition is always satisfied in such a selection. We introduce some strategies to determine an appropriate regularization matrix later in the paper.

\begin{proposition} \label{Proposition_1}
Let $R \in \mathrm{S}_{+}^n, A^T A \in \mathrm{S}_{++}^n$, then the regularization matrix $R$ and the inverse $(A^T A+R)^{-1}$ are simultaneously diagonalizable via similarity by some unitary matrix if and only if $R$ and $(A^T A+R)^{-1}$ commute; if and only if $R\left(A^T A+R\right)^{-1}$ is normal; and if and only if the matrix $R\left(A^T A+R\right)^{-1}$ is a positive semidefinite matrix.
\end{proposition}

\begin{proof}
Take $H=\left(A^T A+R\right)^{-1} \in \mathrm{S}_{++}^n$, then $H$ is diagonalizable via similarity by some unitary matrix. According to \cite{horn2012matrix} (Theorem 1.3.12, p. 62 and Theorem 4.1.6, p. 229), $R$ and $H$ are simultaneously diagonalizable by some unitary matrix if and only if $R$ and $H$ commute. From \cite{horn2012matrix} (Theorem 4.5.15, p. 286) or \cite{meenakshi1999product} (Theorem 3), $R$ and $H$ commute if and only if $RH$ is normal, and if and only if $RH$ is a positive semidefinite matrix.
\end{proof}

Proposition \ref{Proposition_1} shows important and useful properties of the matrix $R\left(A^T A+R\right)^{-1}$ 
and also gives an insight into selecting the regularization matrix $R$.

\subsection{Explanation of regularization methods}
From equation \eqref{Approximate_expression}, one obtains an approximation expression of the inverse of $A^T A$. To approximate the inverse of the matrix $A^T A$, the TR method uses the regularization matrix $\mu^2 I$ in its solution. As the TR solution only contains one term of the approximation expression \eqref{Approximate_expression}, that is the main reason why the TR solution is often overly smooth.

In our proposed method, it uses more than one term of the approximation expression, which enhances the approximate performance of the HR solution. In general, using the first two terms ($k=1$) of the matrix series is a wise choice, which requires a similar computation O($mn^2$) as LS and TR.

One can improve the numerical stability of the ill-conditioned problem described by equation \eqref{problem_1} by calculating its approximate solution using the new matrix $(A^T A+R)$ with a small condition number, which is achieved by selecting a suitable regularization matrix $R$. It is known that matrix inversion is very computationally expensive (O($n^3$)) in matrix operations. The computing result of $(A^T A+R)^{-1}$ can be reused in one calculation of the HR solution. This means that the computation expenses of the HR method and TR method are almost the same. As the new matrix $(A^T A+R)$ can be designed as a sparse symmetric matrix, the approximation expression \eqref{Approximate_expression} also provides a fast and computable approximate solution for calculating the inverse of a matrix.

\subsection{Error bound}
We discuss the error bound \cite{mekoth2021fractional} of the HR method in this section. First, the matrix series is
\begin{equation}\label{Approximation_residual}
\begin{aligned}
F\left(R\left(A^T A+R\right)^{-1}\right):= \sum_{i=k+1}^{\infty}\left(R\left(A^T A+R\right)^{-1}\right)^i,
\end{aligned}
\end{equation}
where $k$ is the same index of the equation \eqref{HR_soluntion}. This is the approximation residual 
of $\left(I-R\left(A^T A+R\right)^{-1}\right)^{-1}$ in regularization solutions. Then the bias between the LS solution and the HR solution is obtained by
\begin{equation}\label{Bias_LS_HR}
\begin{aligned}
\hat{x}_{ls}-\hat{x}^k_{hr}&=\left(A^T A+R\right)^{-1} F\left(R\left(A^T A+R\right)^{-1}\right) A^T b \\
&=\left(A^T A\right)^{-1} \left(R\left(A^T A+R\right)^{-1}\right)^{k+1} A^T b.
\end{aligned}
\end{equation}
The error bound of the estimate HR method is
\begin{equation}\label{Error_bound_HR}
\begin{aligned}
\left\|\hat{x}_{ls}-\hat{x}^k_{hr}\right\| &=\left\|\left(A^T A+R\right)^{-1} F\left(R\left(A^T A+R\right)^{-1}\right) A^T b\right\|.\\
% &=\left\|\left(A^T A\right)^{-1} \left(R\left(A^T A+R\right)^{-1}\right)^{k+1} A^T b\right\|.
\end{aligned}
\end{equation}
For $k=1$, the bias has another similar expression according to a similar approximation expression of the matrix $A^T A$ in the LS solution
\begin{equation}
\begin{aligned}
 &\hat{x}_{ls}-\hat{x}^1_{hr}=\left(A^T A+R\right)^{-1} \sum_{i=2}^{\infty}\left(-R\left(A^T A\right)^{-1}\right)^i A^T b \\
& =\left(A^T A+R\right)^{-1} \\
& \quad \times\left(\left(I+R\left(A^T A\right)^{-1}\right)^{-1}-I+R\left(A^T A\right)^{-1}\right) A^T b, \\
\end{aligned}
\end{equation}
which requires the condition $\rho\left(R\left(A^T A\right)^{-1}\right)<1$  (see \textbf{Appendix B}).

The difference between the TR solution and the LS solution is
\begin{equation}\label{Difference_TR_HR}
\begin{aligned}
&\left\|\hat{x}_{tr}-\hat{x}_{ls}\right\| \\
& = \left\|\left(A^T A+R\right)^{-1} \sum_{i=1}^{\infty}\left(R\left(A^T A+R\right)^{-1}\right)^i A^T b\right\| .
\end{aligned}
\end{equation}

From equation \eqref{Error_bound_HR}, the difference between the LS solution and the HR solution is bounded. As it is known that the LS solution is an unbiased solution for problem \eqref{problem_min_1}, the difference between the LS solution and the HR solution is a bias of the HR solution. The following discussion of the approximation residual and equation \eqref{Bias_LS_HR} show that this bias can be calculated for any given $A$ before the estimation of the solution of the problem \eqref{problem_min_1}, which is very important for some problems, such as mobile robot localization problems discussed in the next section that estimate the solutions with the same given $A$ but with different $b$. From equation \eqref{Difference_TR_HR}, the difference between the TR solution and the HR solution is bounded.

For an HR solution, if $k>0$, the approximate residual $F\left(R\left(A^T A+R\right)^{-1}\right)$ of the HR solution has fewer terms than that in the TR solution. Then according to the triangle inequality of a norm, the error bound of the HR solution is smaller than the error bound of the TR solution (see \textbf{Appendix C}). In other words, the HR solution is closer to the unbiased solution of the ill-conditioned problem under the least squares criterion. From equation \eqref{Approximate_expression}, this is achieved by approximating the inverse of the matrix $A^T A$ using more than one term in the matrix power series, which also means that the over-smoothing problem in TR is solved through matrix inverse approximation.

\subsection{Stability and consistency}

For the matrix $R\left(A^T A+R\right)^{-1}$, if its spectral radius satisfies $\rho\left(R\left(A^T A+R\right)^{-1}\right)<1$, then one has
\begin{equation}
\lim _{i \rightarrow \infty}\left(R\left(A^T A+R\right)^{-1}\right)^i \rightarrow O_{n},
\end{equation}
where $O_{n} \in \mathbb{R} ^{n \times n}$ is a zero matrix.

For the $(k+1)$th-order solution $\hat{x}^{k+1}_{hr}$ and the $k$th-order solution $\hat{x}^k_{hr}$, according to equation \eqref{Bias_LS_HR}, one has
\begin{equation} \label{difference_two_hr_solution}
 \hat{x}^{k+1}_{hr}-\hat{x}^k_{hr}=\left(A^T A+R\right)^{-1}\left(R\left(A^T A+R\right)^{-1}\right)^{k+1} A^T b,
\end{equation}
which follows
\begin{equation}
\lim _{k \rightarrow \infty}\left(\hat{x}^{k+1}_{hr}-\hat{x}^k_{hr}\right) \rightarrow O,
\end{equation}
where $O$ is the zero vector.

Note that according to equation \eqref{Bias_LS_HR}, the bias of a given HR solution is exactly with the factor $R\left(A^T A+R\right)^{-1}$. A larger $k$ ensures a smaller bias, and for $k \rightarrow \infty$ or $R \rightarrow O$, the bias is zero. One is possible to design a matrix $R$ that satisfies $\rho\left(R\left(A^T A+R\right)^{-1}\right)<1$ and ensures that the new matrix $A^T A+R$ is well-conditioned. For some specific selection of $R$, the new matrix $A^T A+R$ becomes a scalar matrix and reaches the smallest condition number, which admits the optimal numerical stability. 

Hence, the proposed HR method provides a stable and consistent solution. In the following subsections, we show that, in some cases, the approximation residual of the matrix inverse has a lower bound and an upper bound. We also introduce how to select $R$ to ensure that $A^T A+R$ and $R\left(A^T A+R\right)^{-1}$ meet the requirements.

\subsection{The approximation residual of the matrix inverse} \label{bounds_approximation_residual}
From the approximation expression \eqref{Approximate_expression} and equation \eqref{Approximation_residual}, the approximation residual of $\left(I-R\left(A^T A+R\right)^{-1}\right)^{-1}$ is $F(R(A^T A+R)^{-1})$ . We show that there is a relationship of the approximation residual
\begin{equation}
\begin{aligned}
F\left(R\left(A^T A+R\right)^{-1}\right) \preceq o^\alpha\left(R\left(A^T A+R\right)^{-1}\right),
\end{aligned}
\end{equation}
where $\alpha$ is a constant, $o^\alpha\left(A\right)= C A^ {\alpha}$ ($C$ is a matrix), the symbol $\preceq$ represents a generalized inequality relationship (partial ordering) \textit{succeeds or equal to} of two matrices \cite{boyd2004convex}, which is similar to the sum of an infinite geometric series.
\begin{lemma}
 If the spectral radius $\rho\left(R\left(A^T A+R\right)^{-1}\right)<1$, then the approximation residual $F\left(R\left(A^T A+R\right)^{-1}\right)$ from expression \eqref{Approximate_expression} which is defined in equation \eqref{Approximation_residual}, is described by two commuting matrices, and the residual is given by
\begin{equation} \label{Approximation_residual_F}
\begin{aligned}
&F\left(R\left(A^T A+R\right)^{-1}\right)\\
&=\left(I-R\left(A^T A+R\right)^{-1}\right)^{-1}\left(R\left(A^T A+R\right)^{-1}\right)^{k+1},
\end{aligned}
\end{equation}
and if the matrix $R\left(A^T A+R\right)^{-1}$ is diagonalizable, then these two commuting matrices are simultaneously diagonalizable matrices.\highlight{}\unhighlight{}
\end{lemma}

\begin{proof}
Rewrite the residual of the approximate expression, then one gets
\begin{equation}
\begin{aligned}
& F\left(R\left(A^T A+R\right)^{-1}\right)=\sum_{i=k+1}^{\infty}\left(R\left(A^T A+R\right)^{-1}\right)^i \\
& =\left(I-\left(R\left(A^T A+R\right)^{-1}\right)\right)^{-1}\left(R\left(A^T A+R\right)^{-1}\right)^{k+1} \\
& =\left(R\left(A^T A+R\right)^{-1}\right)^{k+1}\left(I-\left(R\left(A^T A+R\right)^{-1}\right)\right)^{-1},
\end{aligned}
\end{equation}
which shows that the two factors $\left(I-R\left(A^T A+R\right)^{-1}\right)^{-1}$ and $(R(A^T A+R)^{-1})^{k+1}$ in equation \eqref{Approximation_residual_F} commute. If $R\left(A^T A+R\right)^{-1}$ is diagonalizable, then these two diagonalizable matrices are simultaneously diagonalizable by some nonsingular matrix. 
\end{proof}

\begin{theorem} \label{Theorem_1}
If the matrix $R\left(A^T A+R\right)^{-1}$ is symmetric, then the approximation residual shown in equation \eqref{Approximation_residual_F} has an upper bound and a lower bound given as follows
\begin{equation}\label{Residual_lower_upper_bound}
\begin{aligned}
&\frac{1}{1-\lambda_{\min }} \left(R\left(A^T A+R\right)^{-1}\right)^{k+1} \preceq F\left(R\left(A^T A+R\right)^{-1}\right)\\
&\preceq \frac{1}{1-\lambda_{\max }} \left(R\left(A^T A+R\right)^{-1}\right)^{k+1},
\end{aligned}
\end{equation}
where $\lambda_{\min }$ and $\lambda_{\max }$ are the minimum and the maximum eigenvalues of the matrix $R\left(A^T A+R\right)^{-1}$, and there is always some $R$  such that the approximation method \eqref{Approximate_expression} and the equation \eqref{Residual_lower_upper_bound} are satisfied.\highlight{}\unhighlight{}
\end{theorem}

\begin{proof}
If $R\left(A^T A+R\right)^{-1}$ is symmetric (Proposition \ref{Proposition_1}), then $R\left(A^T A+R\right)^{-1}$ and $\left(I-R\left(A^T A+R\right)^{-1}\right)^{-1}$ are symmetric and diagonalizable. Apply the simultaneous diagonalization to the two matrices on the right-hand side of equation \eqref{Approximation_residual_F}, and let
\begin{equation}
R\left(A^T A+R\right)^{-1}=P \operatorname{diag}\left\{\lambda_{\max }, \ldots, \lambda_{\min }\right\} P^{-1},
\end{equation}
then one has
\begin{equation}
\begin{aligned}
& \left(I-R\left(A^T A+R\right)^{-1}\right)^{-1} \\
& =P \operatorname{diag}\left\{\frac{1}{1-\lambda_{\max }}, \ldots, \frac{1}{1-\lambda_{\min }}\right\} P^{-1}.
\end{aligned}
\end{equation}
Note that $\rho\left(R\left(A^T A+R\right)^{-1}\right)<1$, then one has
\begin{equation}\label{Residual_bound_k0}
\frac{1}{1-\lambda_{\min }} I \preceq\left(I-R\left(A^T A+R\right)^{-1}\right)^{-1} \preceq \frac{1}{1-\lambda_{\max }} I.
\end{equation}
Multiplying $\left(R\left(A^T A+R\right)^{-1}\right)^{k+1}$ on all sides of equation \eqref{Residual_bound_k0}, then one obtain equation \eqref{Residual_lower_upper_bound}. From equation \eqref{HR_solution_k0}, there always exist some regularization matrices that meet the spectral radius condition and ensure that the matrix $R\left(A^T A+R\right)^{-1}$ is symmetric.
\end{proof}
It follows that, for $R \in \mathrm{S}_{+}^n, A^T A \in \mathrm{S}_{++}^n$, the new PD matrix $(A^T A +R)$ is a convex function with respect to the regularization matrix $R$, and if $\left(R\left(A^T A+R\right)^{-1}\right) \in \mathrm{S}_{+}^n$, then $o^{k+1}\left(R\left(A^T A+R\right)^{-1}\right)$ is a convex function with respect to the matrix $R\left(A^T A+R\right)^{-1}$. Note that $(A^T A+R)$ and $R\left(A^T A+R\right)^{-1}$ are simultaneously diagonalizable via congruence by some nonsingular matrix. By some mild relaxations\cite{fuhry2012new}, it is possible to derive equivalent convex problems to optimization problems constructed by these simultaneously diagonalizable constraints \cite{ben2014hidden}, which are easy to solve.

As the approximation matrix inverse is used to approximate the inverse of $A^T A$ in a regularization solution, the upper bound in Theorem \ref{Theorem_1} indicates that the approximation residual is bounded for the regularization solution, while the lower bound shows a computable bias of the regularization solution. The bounds here also suggest that, for a regularization method, the difference between the approximate problem (\ref{TR_ap_problem_min}) and the original problem (\ref{problem_min_1}) is finite and computable.

\subsection{The a priori criterion to improve the numerical stability of the ill-conditioned problem}
We propose an \emph{a priori} criterion to obtain the optimal regularization matrix. The proposed criterion improves the numerical stability of calculating the solution of problem \eqref{problem_min_1}, which ensures a numerically stable behavior of the proposed method. The objective function of the proposed criterion is given as follows
\begin{equation}\label{Priori_criterion}
\begin{aligned} % \begin{gather} % \end{gather}
& \min _{k \in \mathbb{N}, R \in \mathbb{R}^{n\times n}} f_0(k, R) \\
& =\min _{k \in \mathbb{N}, R \in \mathbb{R}^{n\times n}} 
\begin{array}{l}
\bigg\{\left\|\sum_{i=k+1}^{\infty}\left(R\left(A^T A+R\right)^{-1}\right)^i\right\| \\
 \quad + \left\|\left(A^T A+R\right)\right\|\left\|\left(A^T A+R\right)^{-1}\right\| \bigg\}\\
\end{array} 
\\
& \quad\quad\quad \quad\quad\quad\quad\text { s. t. } \quad \rho\left(R\left(A^T A+R\right)^{-1}\right)<1.\\
\end{aligned}
\end{equation}

The first term of the objective function is the approximation residual of $\left(I-R\left(A^T A+R\right)^{-1}\right)^{-1}$, which is the reason for the estimation bias of the high-order regularization method, and the second term is the condition number of the new matrix $A^T A + R$, which describes the numerical stability of the problem. Note that, if $R \in \mathrm{S}_{+}^n$, then the new matrix $A^T A + R$ is a PD matrix, and according to Lemma \ref{Lemma_2}, the spectral radius condition is always satisfied. The proposed criterion considers the trade-off between modifying the small eigenvalues and keeping the information of some dimensions with respect to these small eigenvalues. The proposed \emph{a priori} criterion \eqref{Priori_criterion} is determined by the known matrix $A$ before any estimation of the solution of the problem, taking into account the trade-off between the estimation bias of the high-order regularization method and the improvement in the numerical stability of the approximation solution of the problem \eqref{problem_min_1}, which is achieved by using a new matrix $\left(A^T A+R\right)$ with a small condition number. For any problem described by \eqref{problem_min_1}, the proposed criterion \eqref{Priori_criterion} can be used to find the new (PD) matrix $\left(A^T A+R\right)$ with a small condition number to calculate the approximation solution of the problem.

To determine the regularization matrix $R$, similar to the truncated singular value decomposition (TSVD) method, which is found in \cite{fuhry2012new}, one way to improve the performance of the HR method is to choose some regularization matrix $R$ that only modifies the smallest eigenvalue (or some small eigenvalues) of $(A^T A)$. The regularization matrix is
\begin{equation} \label{Relaxation_R}
R=P \operatorname{diag}\left\{\lambda_{R, 1}, \ldots, \lambda_{R, n}\right\} P^{-1} \in \mathbb{R}^{n \times n},
\end{equation}
where $\lambda_{R, i}, i=1, \ldots, n$ are eigenvalues of the regularization matrix $R$, and $P$ is a unitary matrix whose columns are the eigenvectors of matrix $A^T A$.
\begin{lemma} \label{lemma_relaxation}
For a regularization matrix defined in \eqref{Relaxation_R}, if $\lambda_{R, i}=\max \left\{\mu^2-\lambda_i, 0\right\}, i=1, \ldots, n$ where $0<\lambda_{i+1} \leq \lambda_i, i=1, \ldots, n-1$ are the eigenvalues of the matrix $A^T A$, and $\lambda_{s+1} \leq \mu^2 \leq \lambda_s, 0 \leq s \leq n-1$, then the regularization matrix only modifies some small eigenvalues of $A^T A$ that are smaller than $\lambda_s$; if $\mu^2=\lambda_s$, the new matrix $A^T A+R$ is given as
\begin{equation}\label{New_system_matrix}
A^T A+R=P \operatorname{diag}\left\{\lambda_1, \ldots, \lambda_{s-1}, \lambda_s, \ldots, \lambda_s\right\} P^{-1},
\end{equation}
where $A^T A+R$ is a positive definite matrix, $P$ is a unitary matrix whose columns are the corresponding eigenvectors of matrix $A^T A$, and the regularization matrix $R$ of equation \eqref{New_system_matrix} satisfies the spectral radius condition.
\end{lemma}

\begin{proof}
From equations \eqref{Relaxation_R} and \eqref{New_system_matrix}, the given $R$ and $A^T A$ commute, then they are simultaneously diagonalizable by some unitary matrix $P$. Let
\begin{equation}
A^T A=P \operatorname{diag}\left\{\lambda_1, \ldots, \lambda_s, \lambda_{s+1}, \ldots, \lambda_n\right\} P^{-1},
\end{equation}
and
\begin{equation}
R=P \operatorname{diag}\left\{\max \left\{\mu^2-\lambda_1, 0\right\}, \ldots, \max \left\{\mu^2-\lambda_n, 0\right\}\right\} P^{-1},
\end{equation}
note that $\lambda_{s+1} \leq \mu^2 \leq \lambda_s$, then one obtains the new PD matrix
\begin{equation}
\begin{aligned}
A^T A+R =P \operatorname{diag}\left\{\lambda_1, \ldots, \lambda_s, \mu^2, \ldots, \mu^2\right\} P^{-1}. \\
\end{aligned}
\end{equation}

Note that $0<\lambda_i$ and $\lambda_{s+1} \leq \mu^2 \leq \lambda_s$, then the condition $\rho\left(R\left(A^T A+R\right)^{-1}\right)<1$ is always satisfied.

When choosing $\mu^2 = \lambda_s$, we complete the proof. 
\end{proof}

As a consequence of Lemma \ref{lemma_relaxation}, the high-order regularization improves the numerical stability of the solution for ill-conditioned problems when the regularization matrix is chosen according to the equation \eqref{New_system_matrix}. The numerical stability is enhanced as the condition number of the new matrix $A^TA+R$ is smaller than the condition number of the matrix $A^TA$ when ${\lambda_s}>{\lambda_n}$, and the new matrix $A^TA+R$ has a condition number $\frac{\lambda_1}{\lambda_s}$ with respect to the spectral norm while the condition number of the matrix $A^TA$ is $\frac{\lambda_1}{\lambda_n}$. The regularization matrix $R$ is a function of the variables $\mu$ and $s$, and hence, if $s$ is determined, then the regularization matrix $R$ is a function only with respect to $\mu$. For many robot localization problems, $s$ is able to be determined by identifying the variables that significantly affect the ill-conditioned characteristics of the problem (\ref{problem_1}) or by calculating the eigenvalues of the matrix $A^TA$.

\begin{remark}
For some large-scale ill-conditioned problems, the number of variables is very large, and the determination of $s$ becomes difficult. A general method is to set $s=1$, and then check if the result is acceptable. If the solution is not acceptable, then increase $s$ gradually. A more promising method to determine $s$ is to perform the principal component analysis.
\end{remark}

This restructure \eqref{Relaxation_R} of $R$ provides a more explainable and reasonable regularization for the ill-conditioned problem as it only changes some small eigenvalues (or the smallest eigenvalue) of the matrix $A^T A$. Some methods \cite{fuhry2012new, yang2015modified, cui2020special} for improving the SVD method are also applicable to select the regularization matrix in the proposed high-order regularization method after minor modifications. The regularization matrix under the restructure \eqref{Relaxation_R} makes sure the conditions in the previous theorems of this work are satisfied. However, replacing the small eigenvalues with some $\lambda_s$ directly leads to losing the information on dimensions with respect to these small eigenvalues. Therefore, it is better to consider a balance between modifying the small eigenvalues and keeping the information of some dimensions with respect to these small eigenvalues while choosing a regularization matrix.

Instead of calculating the loss function directly, it is better to modify the loss function \eqref{Priori_criterion} slightly by adding some reasonable restructures on the regularization matrix $R$ according to the above discussion and solve the loss function more easily. The simplified loss function is
\begin{equation}\label{Loss_function_relaxation}
\begin{aligned}
& \min _{s, k \in \mathbb{N} , \mu \in \mathbb{R} , R \in \mathrm{S}_{+}^n} f_1 (s, k, \mu, R)\\
& = \min _{s, k \in \mathbb{N} , \mu \in \mathbb{R} , R \in \mathrm{S}_{+}^n} \Bigg\{\left\|\sum_{i=k+1}^{\infty}\left(R\left(A^T A+R\right)^{-1}\right)^i\right\| \\
& \quad\quad\quad\quad\quad\quad\quad\quad +\left\|\left(A^T A+R\right)\right\|\left\|\left(A^T A+R\right)^{-1}\right\| \Bigg\}, \\
& \text { s. t. } \left\{ \begin{array}{l}
% \rho\left(R\left(A^T A+R\right)^{-1}\right)<1, \\
A^T A=P \operatorname{diag}\left\{\lambda_1, \ldots, \lambda_s, \lambda_{s+1}, \ldots, \lambda_n\right\} P^{-1}, \\
R=P \operatorname{diag}\left\{\lambda_{R, 1}, \ldots, \lambda_{R, n}\right\} P^{-1},\\
\lambda_{R, i}=\max \left\{\mu^2-\lambda_i, 0\right\}, i=1, \ldots, n, \\
\lambda_{s+1} \leq \mu^2 \leq \lambda_{s}, \\
s \leq n=\operatorname{rank}\left(A^T A\right),
\end{array}\right.
\end{aligned}
\end{equation}
where the regularization matrix has a special constraint. From equation \eqref{New_system_matrix}, under these constraints of $R$, it allows the regularization matrix to modify only some small eigenvalues. In general, if considering only modifying the smallest eigenvalue of the matrix $A^T A$, one can calculate the loss function by setting $s=n-1,k=1$. Note that the regularization matrix $R$ is related to $\mu^2$, then only $\mu^2$ needs to be determined among the parameters in the prior criterion.

\begin{theorem} \label{Theorem_3}
For $s=n-1, k=1$, the loss function described in equation 
% \eqref{Loss_function_k1_s_n_1} 
\eqref{Loss_function_relaxation} is simplified as
\begin{equation} \label{Loss_function_mu_k1}
\begin{aligned}
&\min _{\mu^2 \in \mathbb{R}} f\left(\mu^2\right)=\min _{\mu^2 \in \mathbb{R}}\left\{\frac{\left(\mu^2-\lambda_n\right)^2}{\lambda_n \mu^2}+\frac{\lambda_1}{\mu^2}\right\}, \\
& \text { s. t. } \lambda_n \leq \mu^2 \leq \lambda_{n-1}, 0<\lambda_i,
\end{aligned}
\end{equation}
where $f\left(\mu^2\right)$ is a convex function with respect to $\mu^2$, and the loss function \eqref{Loss_function_mu_k1} has a unique minimum at
\begin{equation}
\mu^2= \min \left\{\sqrt{\lambda_n^2+\lambda_n \lambda_1},\lambda_{n-1}\right\}.
\end{equation}
\end{theorem}

\begin{proof}
From equation \eqref{Loss_function_relaxation}, if $s=n-1$, one has
\begin{equation}
\begin{aligned}
\left(A^T A+R\right) &=P \operatorname{diag}\left\{\tilde{\lambda}_1, \ldots, \tilde{\lambda}_n\right\} P^{-1}\\
&=P \operatorname{diag}\left\{\lambda_1, \ldots, \lambda_{n-1}, \mu^2\right\} P^{-1},
\end{aligned}
\end{equation}
where $\tilde{\lambda}_i=\max \left\{\lambda_i, \mu^2\right\}, i=1, \ldots, n$ is the eigenvalue of the matrix $\left(A^T A+R\right), 0<\tilde{\lambda}_i$, and $P$ is some unitary matrix.

The condition number of the special new PD matrix $A^T A+R$ is
\begin{equation} \label{new_condition_number}
\begin{aligned}
\kappa\left(A^T A+R\right) &=\left\|\left(A^T A+R\right)\right\|\left\|\left(A^T A+R\right)^{-1}\right\| \\
&=\frac{\tilde{\lambda}_1}{\tilde{\lambda}_n}=\frac{\lambda_1}{\mu^2}.
\end{aligned}
\end{equation}

For the residual, if $k=1$, one has
\begin{equation}
\begin{aligned}
& \sum_{i=2}^{\infty}\left(R\left(A^T A+R\right)^{-1}\right)^i \\
& =\left(I-R\left(A^T A+R\right)^{-1}\right)^{-1}\left(R\left(A^T A+R\right)^{-1}\right)^2. \\
\end{aligned}
\end{equation}
For the factors in the first term of the equation \eqref{Loss_function_relaxation}, the matrix power series is rewritten as
\begin{equation}
\begin{aligned}
R\left(A^T A+R\right)^{-1}=P\left(\operatorname{diag}\left\{0, \ldots, 0, \frac{\mu^2-\lambda_n}{\mu^2}\right\}\right) P^{-1},
\end{aligned}
\end{equation}
and one has
\begin{equation}\label{eq_inverse_IRAAR}
\begin{aligned}
\left(I-R\left(A^T A+R\right)^{-1}\right)^{-1}=P\left(\operatorname{diag}\left\{1, \ldots, 1, \frac{\mu^2}{\lambda_n}\right\}\right) P^{-1}. \\
\end{aligned}
\end{equation}
Hence, for the first term in the loss function \eqref{Loss_function_relaxation}, one has
\begin{equation} \label{residual_k1}
\begin{aligned}
\left\|\sum_{i=2}^{\infty}\left(R\left(A^T A+R\right)^{-1}\right)^i\right\| =\frac{\left(\mu^2-\lambda_n\right)^2}{\lambda_n \mu^2}.
\end{aligned}
\end{equation}
Substitute equations \eqref{new_condition_number} and \eqref{residual_k1} into the equation 
\eqref{Loss_function_relaxation}, then one has the loss function \eqref{Loss_function_mu_k1}. 

Consider the second derivative of $f\left(\mu^2\right)$, then it follows that $f\left(\mu^2\right)$ is a convex function with respect to $\mu^2$. Hence, for the loss function \eqref{Loss_function_mu_k1}, there must be a unique minimum for this function. It follows that
\begin{equation}
\mu^2= \min \left\{\sqrt{\lambda_n^2+\lambda_n \lambda_1},\lambda_{n-1}\right\},
\end{equation}
which is the solution for equation \eqref{Loss_function_mu_k1}. 
\end{proof}
Similarly, for $s=n-1,k=0$, the loss function \eqref{Loss_function_relaxation} is simplified as
\begin{equation}
\begin{aligned}
& \min _{\mu^2 \in \mathbb{R}} f\left(\mu^2\right)=\min _{\mu^2 \in \mathbb{R}}\left\{\frac{\mu^2-\lambda_n}{\lambda_n}+\frac{\lambda_1}{\mu^2}\right\},\\
&\text { s.t. } \lambda_n \leq \mu^2 \leq \lambda_{n-1}, 0<\lambda_i.
\end{aligned}
\end{equation}
Then, the corresponding optimal regularization parameter is

\begin{equation} \label{miu2_k0}
\mu^2 = \left\{ \begin{array}{ll}
\min \left\{\sqrt{\frac{2 \lambda_1}{\lambda_n}},\lambda_{n-1}\right\}, & \text{if }\sqrt{\frac{2 \lambda_1}{\lambda_n}} \geq \lambda_n\\
\max \left\{\sqrt{\frac{2 \lambda_1}{\lambda_n}},\lambda_{n}\right\}, & \text{if }\sqrt{\frac{2 \lambda_1}{\lambda_n}} \leq \lambda_{n-1}
\end{array}\right.
\end{equation}
which provides the optimal regularization matrix $R$ of the TR method under the proposed \emph{a priori} criterion.

Additionally, for any given $k\in \mathbb{N}, s=n-1$, the loss function is a convex function with respect to $\mu^2$, then there must be a unique minimum for the loss function \eqref{Loss_function_relaxation}, and the optimal regularization matrix always exists (see \textbf{Appendix D}).

For robot localization problems, Theorem \ref{Theorem_3} provides a direct way to calculate the regularization matrix $R$ rather than an empirical method in most applications of regularization methods.

\subsection{Improved regularization solution}
In this section, we propose two techniques to improve the proposed HR method for some applications in robot problems. In most robot applications, such as mobile robot localization \cite{wang2017ultra,qin2018vins} and calibration of LiDAR-IMU Systems \cite{lv2022observability}, the estimation of the solution includes a multiple estimation process using a number of measurements. It is known that the LS solution is an unbiased estimator for variables with zero mean white noise, and from equation \eqref{Bias_LS_HR}, the difference between the LS solution and the high-order solution is computable in every calculation. Hence, one can calculate the bias of the HR method by calculating the difference between the LS solution and the HR solution. As matrix $A^T A$  is nonsingular, and if one has $l$ measurements of $b$, which are a significant amount of the data with zero mean white noise, then one can calculate the bias of the HR method as
\begin{equation}
\begin{aligned}
\Delta x & =\frac{1}{l} \sum_{t=1}^l\left(\hat{x}_{hr, t}-x_t^*\right) \\
&\approx \frac{1}{l} \sum_{t=1}^l\left(\hat{x}_{hr, t}-x_t^*\right)-\frac{1}{l} \sum_{t=1}^l\left(\hat{x}_{ls, t}-x_t^*\right) \\
& =\frac{1}{l} \sum_{t=1}^l\left(\hat{x}_{hr, t}-\hat{x}_{ls, t}\right),
\end{aligned}
\end{equation}
where $\hat{x}_{hr, t}$ is the $t$th calculation result of the HR method using the $t$th measurement of $b, \hat{x}_{ls, t}$ is the $t$th calculation result of the LS method, and $x^*$ is the true value. The expression works because the LS solution is an unbiased estimator for a random variable with zero mean white noise, and the difference between the LS solution and the high-order solution is computable. Then a bias-corrected HR solution is given by
\begin{equation}
\begin{aligned}
\hat{x}^k_{hr,l}=\left(A^T A+R\right)^{-1} \sum_{i=0}^k\left(R\left(A^T A+R\right)^{-1}\right)^i A^T b-\Delta x .
\end{aligned}
\end{equation}

One equivalent way to calculate the bias of the high-order regularization solution is calculating the approximation residual \eqref{Approximation_residual_F}. However, as found in equation \eqref{eq_inverse_IRAAR}, if the problem is ill-conditioned, then basically, the matrix $\left(I-R\left(A^T A+R\right)^{-1}\right)^{-1}$ is ill-conditioned. Hence, it is not able to improve the estimation performance of the high-order regularization method in one calculation. Some similar bias-correction methods are found in \cite{ji2022adaptive}. Note that this bias correction method requires a significant amount of data to estimate the bias.

Let $R \in \mathrm{S}_{+}^n$, according to Proposition \ref{Proposition_1}, $R\left(A^T A+R\right)^{-1}$ is symmetric positive semidefinite. From equation \eqref{Residual_lower_upper_bound}, if the matrix $R\left(A^T A+R\right)^{-1}$ is symmetric, then there is an adjustment parameter $\omega$ such that
\begin{equation} \label{find_residual_approximation}
\min _{\lambda_{\min } \leq \omega \leq \lambda_{\max }} \left\| G \left( \omega \right) \right\|_F^2, \text{if }R \in \mathrm{S}_{+}^n ,
\end{equation}
where 
\begin{equation}
\begin{aligned}
&G \left( \omega \right)\\
&=F\left(R\left(A^T A+R\right)^{-1}\right)-\frac{1}{1-\omega}\left(R\left(A^T A+R\right)^{-1}\right)^{k+1},
\end{aligned}
\end{equation}
which shows that the residual is approximated by a convex function with respect to the PSD matrix $R\left(A^T A+R\right)^{-1}$. Here $\| \cdot \|_F$ is the Frobenius norm, and $\lambda_{\max }$ and $\lambda_{\min }$ are the maximum and the minimum eigenvalues of matrix $R\left(A^T A+R\right)^{-1}$. As discussed in Section \ref{bounds_approximation_residual}, the approximation residual has two bounds, both relating to the factor $\left(R\left(A^T A+R\right)^{-1}\right)^{k+1}$. Hence, one can find a matrix \cite{horn2012matrix, fuhry2012new} with the factor $\left(R\left(A^T A+R\right)^{-1}\right)^{k+1}$ to approximate the residual according to \eqref{find_residual_approximation}.

As a possible improvement, if the performance of the high-order regularization method for a given $k$ still needs to be improved, from equation \eqref{Residual_lower_upper_bound}, one way to improve it is to calculate the following solution
\begin{equation}
\begin{aligned}
&\hat{x}^k_{hr,\omega}=\\
&\left(A^T A+R\right)^{-1} \sum_{i=0}^k\left(R\left(A^T A+R\right)^{-1}\right)^i A^T b\\
& +\frac{1}{1-\omega} \left(A^T A+R\right)^{-1} \left(R\left(A^T A+R\right)^{-1}\right)^{k+1} A^T b, \\
&0 \leq \omega \leq \lambda_{\max },
\end{aligned}
\end{equation}
which is an improved solution of the high-order solution. For $k=0$, one has
\begin{equation}
\hat{x}^0_{hr,\omega}=\frac{1}{1-\omega}\left(A^T A+R\right)^{-1} A^T b, 0 \leq \omega \leq \lambda_{\max },
\end{equation}
where $\omega$ is an adjustment parameter based on the performance of the $k$th-order solution. To determine the adjustment parameter, some experimental approaches introduced in \cite{fuhry2012new, hansen1993use} are helpful. There are some researches that introduce a similar adjustment parameter for each diagonal item of the $A^TA$ \cite{yang2015modified, cui2020special}. However, for robot localization problems, if the performance of the high-order regularization is sufficient, for the sake of simplicity, the suggested adjustment parameter is the maximum eigenvalue $\lambda_{\max}$, the minimum eigenvalue $\lambda_{\min}$ or 0.

\section{Robot localization in 3D environment} \label{robot_localization_sliding_window}
The restricted mobile robot localization problems using sensor networks are often found to be ill-conditioned inverse problems, especially when the trajectory of the robot is random and the placement of the sensors is limited.
\begin{figure}[!ht]
\centering
\includegraphics[width=3.5in]{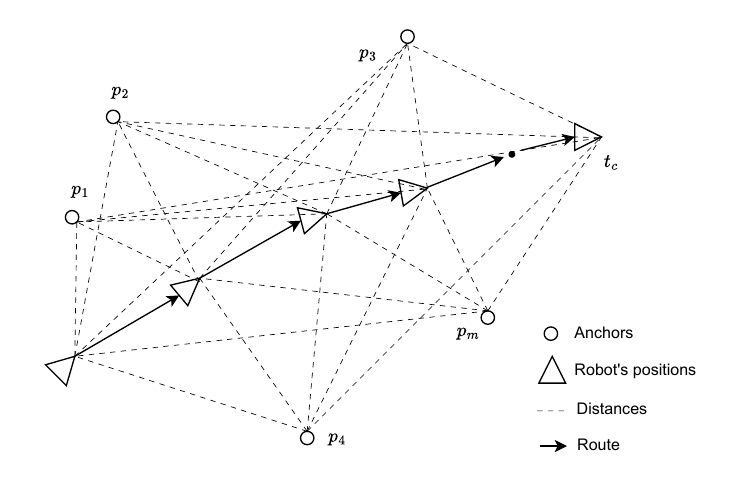}
\caption{Illustration of robot localization problem using anchors.}
\label{robot_localization_using_anchors}
\end{figure}
A robot localization process is shown in Fig. \ref{robot_localization_using_anchors}. In Fig. \ref{robot_localization_using_anchors}, a robot moves continuously along a route. The distances between anchors and the robot are given by a tag, which is a distance measurement sensor equipped with the robot. The positions of the robot in different instants are obtained by using distance measurements. Considering the distance $d_i$ between the robot and the $i$th anchor, one has
\begin{equation}
d_i^2=(x_i-x_r)^2+(y_i-y_r)^2+(z_i-z_r)^2,i=1,\cdots,m+1,
\end{equation} where $\left(x_r, y_r, z_r\right)^T$ is the position of the mobile robot and $\left(x_i, y_i, z_i\right)^T, i=1, \ldots, m+1$ are the positions of $m+1$ anchors with known given positions.
Subtracting $d_i^2 (i \neq m+1)$ and $d_{m+1}^2$, the mobile robot localization problem is formulated into the form as \eqref{problem_1} with
\begin{equation} \label{localization_expression}
\begin{aligned}
A &=\left[\begin{array}{c}
\left(p_1-p_{m+1}\right)^T \\
\vdots \\
\left(p_m-p_{m+1}\right)^T
\end{array}\right], \\
b &=\frac{1}{2}\left[\begin{array}{c}
p_1^T p_1-p_{m+1}^T p_{m+1}+d_{m+1}^2-d_1^2 \\
\vdots \\
p_m^T p_m-p_{m+1}^T p_{m+1}+d_{m+1}^2-d_m^2
\end{array}\right],
\end{aligned}
\end{equation}
where $x=\left(x_r, y_r, z_r\right)^T$ is the position of the mobile robot,
$A \in \mathbb{R}^{m \times 3}$, and  $p_i=\left(x_i, y_i, z_i\right)^T, i=1, \ldots, m+1$. To simplify, take $b_i = p^T_ip_i - p^T_{m+1}p_{m+1} + d_{m+1}^2 - d_i^2 + \epsilon_i$, where $\epsilon_i$ is a zero mean white noise. Then one has $Ax-b=\epsilon$ (see \cite{sayed2005network} for more details). Note that the normal matrix $A^T A$ is expressed as
\begin{equation} \label{ATA_D_xyz}
\begin{aligned}
A^T A &=\sum_{i=1}^m \left(p_i-p_{m+1}\right) \left(p_i-p_{m+1}\right)^T \\
&=\left[\begin{array}{ccc}
\Delta x_a^T \Delta x_a  \quad \Delta x_a^T \Delta y_a  \quad \Delta x_a^T \Delta z_a \\
\Delta y_a^T \Delta x_a  \quad \Delta y_a^T \Delta y_a  \quad \Delta y_a^T \Delta z_a \\
\Delta z_a^T \Delta x_a  \quad \Delta z_a^T \Delta y_a  \quad \Delta z_a^T \Delta z_a  
\end{array}\right],
\end{aligned}
\end{equation}
where \\
\begin{equation}
\begin{aligned}
 &\Delta x_a= \left(x_1-x_{m+1}, \quad\dots, \quad x_m-x_{m+1}\right)^T, \\
 &\Delta y_a= \left(y_1-y_{m+1}, \quad\dots, \quad y_m-y_{m+1}\right)^T,\\
 &\Delta z_a= \left(z_1-z_{m+1}, \quad\dots, \quad z_m-x_{m+1}\right)^T,
\end{aligned}
\end{equation}
and 
$A = \left[\Delta x_a, \Delta y_a,\Delta z_a \right]$. From equation \eqref{ATA_D_xyz}, if some column of $A$ is much smaller than the other columns, for example, $\Delta z_a$ is much smaller than the other two columns, there are some small entries in $A^T A$, which causes the smallest eigenvalue of $A^T A$. This situation often arises in indoor robot localization problems as the height of a room and the width of a narrow corridor in a building are fixed and much smaller than the value in the other dimensions. In these cases, the localization problems become ill-conditioned problems. Therefore, the position estimation of the robot is better calculated using the high-order regularization method.

\begin{algorithm}[!ht]
\caption{Bias-correction of the HR solution with a sliding window.}\label{alg:sliding_window_bias}
\begin{algorithmic}
\STATE 
\STATE {\textsc{Bias sliding window}}$(W)$ 
\STATE \hspace{0.5cm}$ \textbf{if } t_c-t_0 < l $
\STATE \hspace{0.5cm}$ W \gets \Delta x(t) = \hat{x}_{hr, t}-\hat{x}_{ls, t} $ \textbf{ for } $ t = t_0,...,t_c $
\STATE \hspace{0.5cm}$ \Delta x = 0 \text{ }(\text{or } \Delta x =\Delta x(t))$
\STATE \hspace{0.5cm}$ \textbf{else if } t_c-t_0 \geq l $
\STATE \hspace{0.5cm}$ W \gets \Delta x(t) =\hat{x}_{hr, t}-\hat{x}_{ls, t} $ \textbf{ for } $ t = t_c-l+1,...,t_c $
\STATE \hspace{0.5cm}$ \Delta x = \frac{1}{l} \sum \left(\Delta x(t)\right),\Delta x(t) \in W $
\STATE \hspace{0.5cm}\textbf{return}  $\Delta x $
\STATE 
\STATE {\textsc{HR Solution}}$(\hat{x}_{hr})$
\STATE \hspace{0.5cm}$ \hat{x}_{hr} \gets  \hat{x}_{hr} - \Delta x $
\STATE \hspace{0.5cm}\textbf{return}  $\hat{x}_{hr}$
\end{algorithmic}
\label{alg1}
\end{algorithm}
To obtain the bias of the HR solutions, one way is to calculate the mean of the difference between every HR solution and LS solution from the start of the estimation process $t=t_0$ to the current instant $t=t_c$. However, to make sure the estimation of the bias of HR solutions is as accurate as possible, we propose an algorithm based on a sliding window to calculate the difference between the HR solutions and LS solutions. By using the sliding window $W$, which includes $\Delta x(t), t = t_c-l+1,...,t_c $, we consider the mean of differences $\Delta x(t)$ between $l$ HR solutions and $l$ LS solutions from instant $t=t_c-l+1$ to the current instant $t=t_c$. Therefore, we only need to consider the latest $l$ measurements of $b$ and assume that in this time interval, the error of $b$ is a zero mean white noise and $l$ is large enough. The bias correction algorithm based on a sliding window is given as Algorithm \ref{alg:sliding_window_bias}.

A range of problems, whose observation set and variable set are used to construct nonlinear residual functions \cite{bertoni2022perceiving,qin2018vins,li2022diversified,han2023rda,dümbgen2023safe,dong2023trajectory,guizilini2022learning,lv2022observability}, are solved by the Gauss-Newton method. Consider that the object function of the problems constructed from the residual functions is described as
\begin{equation}\label{slam_residual_problem}
\min _{x} \frac{1}{2} \|r(x)\|^2_Q,
\end{equation}
where the variables $x$ are poses, positions of the robot, or any other parameters of interest, $\|x\|^2_Q= x^TQx$, and $Q$ is a weight matrix, normally the inverse of the covariance matrix associated with the measurement noise. By leveraging the Gauss-Newton method, the problems \eqref{slam_residual_problem} are solved by solving the following problem
\begin{equation}\label{slam_residual_problem_linearized}
\min _{x} \frac{1}{2} \|r(x)+ J \delta x\|^2_Q,
\end{equation}
which is simplified as
\begin{equation}\label{weight_problem_min}
J^TQJ \delta x = - J^TQ r(x),
\end{equation}
where $J$ is the Jacobian of $r(x)$ with respect to $x$, and $\delta x$ is the perturbation of $x$. By setting $A^TA=J^TQJ, b= -J^TQ r(x)$, this weighted least square problem is described in a similar form to the problem \eqref{problem_1} and solved after obtaining $\delta x$. The problem is ill-conditioned if $J^TQJ$ is an ill-conditioned matrix, which generally implies that $J$ is ill-conditioned. One method to calculate a new covariance matrix is to recover it with the matrix inverse approximation $(A^TA+R)^{-1} \sum_{i=0}^{k} \left( R(A^TA+R)^{-1}\right)^i$, similar to LS and TR (see \textbf{Appendix E}). As these problems reduce to a form similar to the problem \eqref{problem_min_1}, we focus on analyzing the ill-conditioned situations and presenting our novel solution for solving ill-conditioned problems, using the example of robot localization with a sensor network described in this section.

\section{Simulations and Experiments} \label{simulation_experiment}
To demonstrate the performance of the proposed methods, we compare the results of the proposed HR method with the LS method and some regularization methods. These regularization methods include the TR method and Fuhry's Tikhonov regularization (FTR) method with the optimal regular parameter (OFTR). The OFTR method is an improved TR method and also a special case of HR with $\mu$ provided by the equation \eqref{miu2_k0}. As the FTR bridges TR and TSVD with a filter factor\cite{fuhry2012new}, the relationships between TSVD and the proposed HR are also found through such a similar filter factor. In addition, many of the modified Tikhonov regularizations\cite{yang2015modified,cui2020special,noschese2016some,mohammady2020extension} use a similar form to FTR only with a different filter factor. Hence, the relationships between these modified Tikhonov regularizations and the proposed HR method are discovered. Considering $s=n-1$, some special cases of the proposed HR method are shown in Table \ref{tab:hrm_special_cases}. As these special cases and TSVD are all the fundamental regularization methods and, more importantly, their regularization matrix or regularization parameters can be determined by \emph{a priori} criterion, i.e., existing a closed-form solution, we then compare the performance of the proposed HR method with such regularization methods. We also provide the optimal regularization parameter $\mu$ of the proposed HR method in the special case for $k=1$ in Table \ref{tab:hrm_special_cases}, which is used to compare with the other regularization methods in simulations and experiments.

\begin{table}[!ht]
\caption{Some special cases of the proposed HR method \label{tab:hrm_special_cases}}
\centering
\begin{tabular}{c|c|c|c}
\hline
Method & $k$     & $R$      & $\omega$\\
\hline
LS     & Inf     & 0        & -\\
\hline
TR & 0 & $R = \mu ^2 I$  & 0\\
\hline
FTR     & 0  & $\lambda_{R,n}= \max \left\{\mu^2-\lambda_{n},0 \right\}$ & 0\\
\hline
OFTR    & 0  & $\mu^2=(2\lambda_1 /\lambda_n)^{1/2}$  & 0\\
\hline
HR  & $\geq 0$  & $\mu^2=(\lambda_n^2+\lambda_n \lambda_1)^{1/2}(k=1)$  & 0 , $\lambda_{max}$, or $\lambda_{min}$\\
\hline
\end{tabular}
\end{table}

\subsection{Simulations}

To evaluate the results of different methods, we calculate the root mean square error (RMSE) for the position estimations and position estimations in different dimensions:

\begin{equation}
\text{RMSE}=\sqrt{\frac{1}{l} \sum_{t=1}^l{e_{t}^2}}, \,\, \text{RMSE}_z=\sqrt{\frac{1}{l} \sum_{t=1}^l{e_{z,t}^2}},
\end{equation}
where $e_{t} = \|\hat{x}_t-x_t^*\|$, $x_t = (x_{r,t},y_{r,t}, z_{r,t})$ is the $t$th position of robot, $\hat{x}_t$ and $x_t ^*$ are the concoresponding estimation and true value, $e_{z,t} = (\hat{z}_{r,t}-z_{r,t}^*)$ (same for $\text{RMSE}_x$ and $\text{RMSE}_y$).

The anchors for all simulations are located at (0, 0, 0), (6, 0, 0), (0, 5, 0), (3.5, 3, 0), and (3, 2.5, 0.5). We first create a 3D scenario for random point localization to show the performance of the proposed high-order method. In these simulations, we create 1000 points randomly representing the robot positions. We add zero mean random noise with $\sigma = 0.1 m$ to each distance measurement according to the results in \cite{jimenez2016comparing}, and then each entry of $b$ is with random noise. For this 3D localization problem, the results shown in Table \ref{tab:hr_results}, Table \ref{tab:unbias_hr_results} and Fig. \ref{new-TR_Ri-optimal_k1_unbias-j200_point} indicate that in terms of RMSE, the performance of the proposed HR method is significantly superior to the LS method, the TR method, and Fuhry's Tikhonov regularization (FTR) method with the optimal regularization parameter (OFTR), which is provided by the equation \eqref{miu2_k0}.

\begin{table}[!t]
\caption{Results of different methods (unit: meter)\label{tab:hr_results}}
\centering
\begin{tabular}{c|c|c|c|c}
\hline
Method & LS  & TR  & HR$_\omega$ & HR$_{\lambda}$ \\
\hline
 RMSE & 0.64331 & 0.42686 & 0.39173 & 0.32005\\
\hline
\end{tabular}
\end{table}

\begin{figure*}[!t]
\centering
\subfloat[]{\includegraphics[width=3.5in]{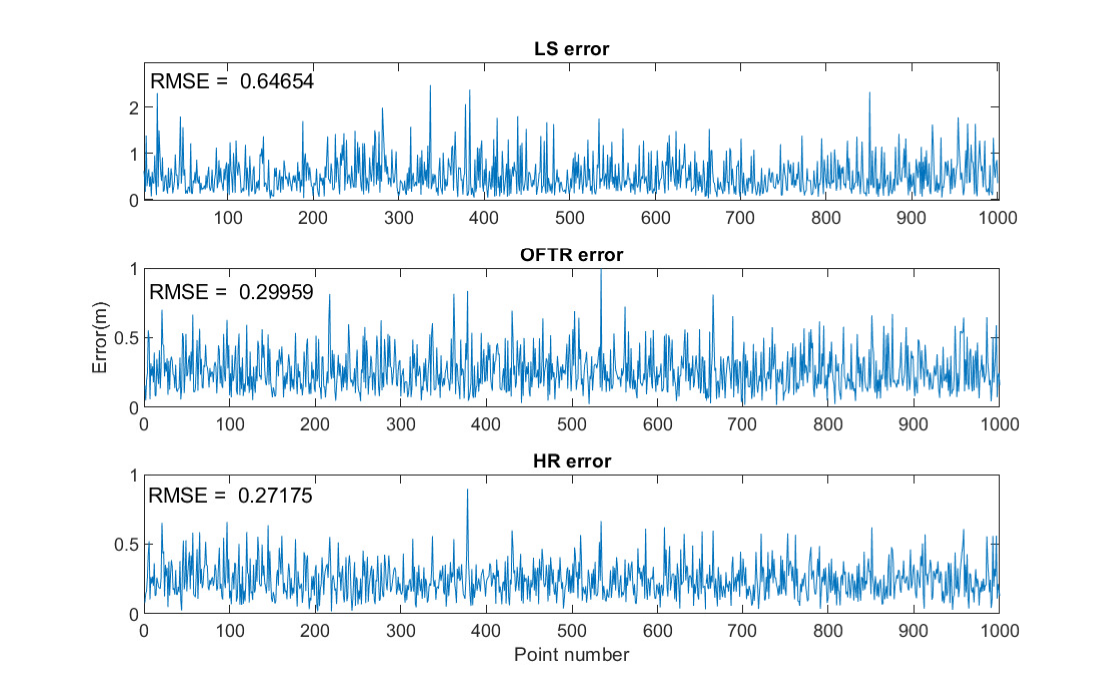}
\label{fig_Er_Ri_optimal_k1_unbias}}
\hfil
\subfloat[]{\includegraphics[width=3.5in]{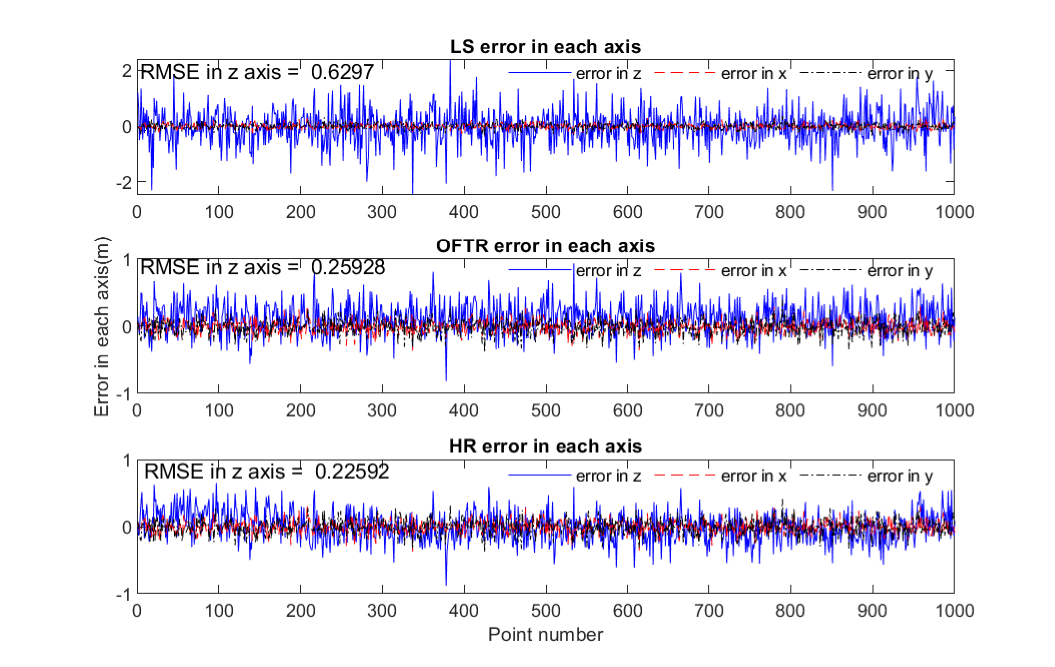}%
\label{fig_Er_Ri_optimal_k1_unbias_zErr}}
\caption{The performance of the high-order method with $R$ only changing the smallest eigenvalue of $A^TA$ and $\mu^2=(\lambda_n^2+\lambda_n \lambda_1)^{1/2} (s=n-1, k=1)$, the OFTR method and the LS method. The bias of the HR solution is corrected in real-time according to the LS solution. The result of the OFTR solution is provided by the FTR method whose regularization parameter $\mu^2=(2\lambda_1 /\lambda_n)^{1/2}$ is given by the \emph{a priori} criterion proposed in this paper for $s=n-1,k=0$, which only changes the smallest eigenvalue of the matrix $A^TA$. (a) Total RMSE. (b) RMSE in different dimensions.}
\label{new-TR_Ri-optimal_k1_unbias-j200_point}
\end{figure*}

\begin{table}[!t]
\caption{Regularization methods with optimal parameter (unit: meter)\label{tab:unbias_hr_results}}
\centering
\begin{tabular}{c|c|c}
\hline
Method & RMSE  & RMSE$_z$  \\
\hline
LS & 0.64654 & 0.62970 \\
\hline
OFTR($\mu^2=(2\lambda_1 /\lambda_n)^{1/2}$) & 0.29959 & 0.25928 \\
\hline
HR($\mu^2=(\lambda_n^2+\lambda_n \lambda_1)^{1/2}$) & \textbf{0.27175} & \textbf{0.22592} \\
\hline
\end{tabular}
\end{table}

In Table \ref{tab:hr_results}, the TR method is with the regularization matrix $R = \mu ^2 I$, the HR method with the same regularization matrix and an adjustment parameter $\omega = \lambda_{min}$ is denoted as HR$_{\omega}$, and the HR method with only changing the smallest eigenvalue of $A^TA$ with $\lambda_{R,n}=\lambda_{n-1}$ is denoted as HR$_\lambda$. From Table \ref{tab:hr_results}, the adjustment parameter and the restructure of the regularization matrix help to improve the behavior of regularization. It is known from the simulation results of Fig. \ref{new-TR_Ri-optimal_k1_unbias-j200_point} and Table \ref{tab:unbias_hr_results} that the high-order regularization method is superior to the OFTR method, which is the special case of the HR method for $k=0$, indicating the order $k$ of regularization benefits the overall behaviors of regularization. The estimation error in $z$ dimension also shows that the high-order regularization method with the proposed \emph{a priori} criterion finds a balance between modifying the small eigenvalues and keeping the information of the dimension with respect to the smallest eigenvalue as its performance in the $z$ dimension is improved.

We also perform simulation and calculate the normalized estimation error squared (NEES) to estimate the consistency of the proposed methods\cite{bar2004estimation}. The results in Table \ref{tab:nees_test} show that the proposed HR estimator with bias correction is consistent, similar to the LS estimator, as their NEES results are approximately equal to the degree of freedom of the variables. The $95\%$ confidence interval of NEES is $[2.8496,3.2428]$, and the results of LS and HR are within the interval. Since NEES is sensitive to the assumption of a standard normal distribution of the estimated errors, the results exhibit applicability in different methods, as observed in the NEES of OFTR, which provides a biased solution.

\begin{table}[!t]
\caption{Normalized estimation error squared\label{tab:nees_test}}
\centering
\begin{tabular}{c|c|c|c}
\hline
Method & LS  & OFTR  & HR \\
\hline
 NEES & 3.0171 & 20.085 & 2.9288 \\
\hline
\end{tabular}
\end{table}

To demonstrate the performance improvement of the high-order regularization method in the dimension ($z$) with respect to the smallest eigenvalue of $A^TA$ and the trade-off capability of the proposed optimal regularization matrix for the estimation bias, over-smoothness and the numerical stability, we create a 3D scenario in which the robot moves along a given route shown in Fig. \ref{given_route}. In this scenario, the robot moves along a given route. The given route consists of two main parts. In the first part, the robot moves on a horizontal plane, and in the second part of the movement, the robot moves on a slope. We also add zero mean random noise with $\sigma = 0.1 m$ to each distance measurement as in the previous simulations.

\begin{figure}[!t]
\centering
\includegraphics[width=3.5in]{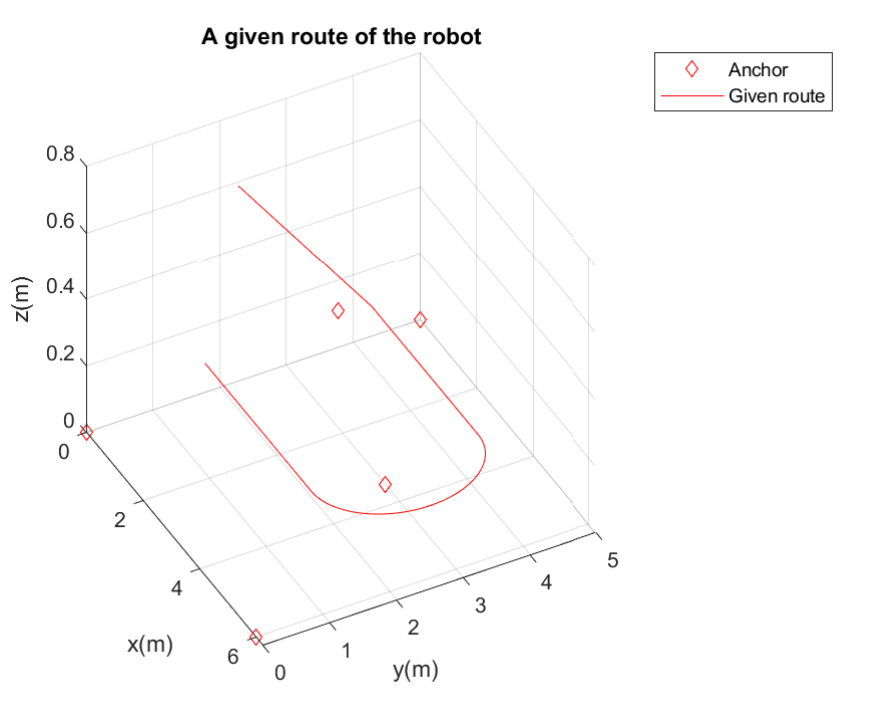}
\caption{A given route to simulate the robot moves in a real 3D environment.}
\label{given_route}
\end{figure}

To show the over-smoothness of some regularization solutions (such as TSVD) and show the performance of the HR method after the bias correction, we perform another simulation based on Algorithm \ref{alg:sliding_window_bias}. The results of different methods are shown in Fig. \ref{new-TR_Ri-optimal_k1_unbias-j200_swin_route} and Table \ref{tab:route_unbias_hr_swin_results}. In this experiment, the bias of the HR solution is corrected by the mean of the difference between the HR solution and the LS solution based on a sliding window.

\begin{figure*}[!htbp]
\centering
\subfloat[]{\includegraphics[width=3.5in]{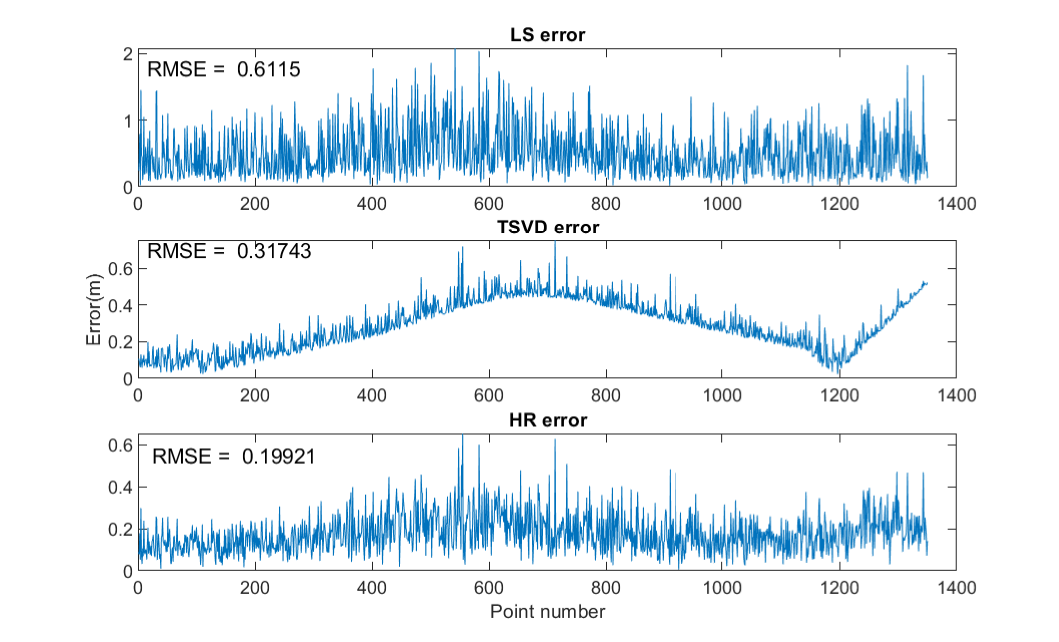}
\label{fig_Er_Ri_optimal_k1_unbias_swin_route}}
\hfil
\subfloat[]{\includegraphics[width=3.5in]{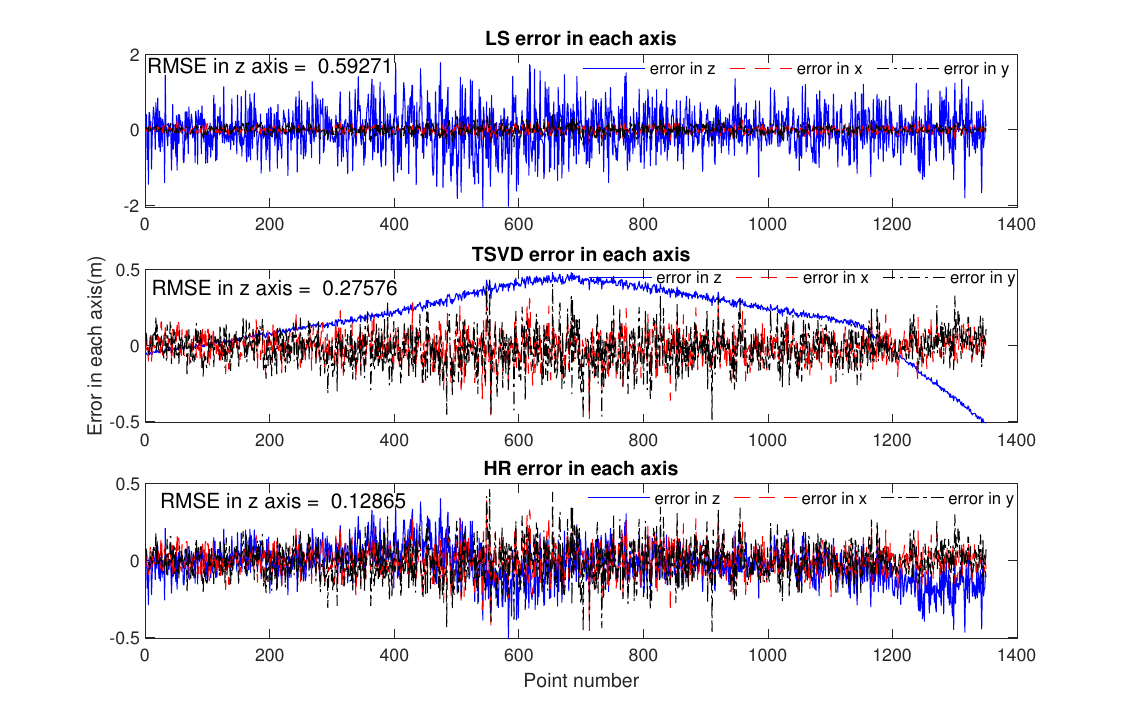}
\label{fig_Er_Ri_optimal_k1_unbias_swin_route_zErr}}
\caption{The performance of the high-order method with $R$ only changing the smallest eigenvalue of $A^TA$ and $\mu^2=\lambda_{n-1}(s=n-1,k=1)$, the TSVD method and LS method. The bias of the HR solution is corrected in real-time based on Algorithm \ref{alg:sliding_window_bias} by using the latest $l$ measurements with a sliding window. The TSVD method truncates the smallest singular value with respect to the z component. In this simulation, the robot moves along a given route. (a) Total RMSE. (b) RMSE in different dimensions.}
\label{new-TR_Ri-optimal_k1_unbias-j200_swin_route}
\end{figure*}

\begin{table}[!t]
\caption{Results of different methods and the bias-corrected HR solution based on Algorithm \ref{alg:sliding_window_bias}(unit: meter) \label{tab:route_unbias_hr_swin_results}}
\centering
\begin{tabular}{c|c|c|c|c}
\hline
Method & RMSE  & RMSE$_x$ & RMSE$_y$ & RMSE$_z$ \\
\hline
LS & 0.61150 & 0.09077 & 0.11994 & 0.59271 \\
\hline
TSVD & 0.31743 & 0.09905 & 0.12209 & 0.27576 \\
\hline
HR & \textbf{0.19921} & 0.09525 & 0.11857 & \textbf{0.12865} \\
\hline
\end{tabular}
\end{table}

From Fig. \ref{new-TR_Ri-optimal_k1_unbias-j200_swin_route} and the numerical result (Table \ref{tab:route_unbias_hr_swin_results}), the RMSE of the HR solution in $z$ dimension is almost the same with the RMSE in $y$ and $x$ dimension, which means that, for such an ill-conditioned robot localization problem, the HR method gives a numerically stable solution just as it does for the well-conditioned problems. In comparison, even though the result of the TSVD performs better than the LS method, the over-smoothness of the TSVD is severe, which is found on the error in $z$ dimension from Fig. \ref{new-TR_Ri-optimal_k1_unbias-j200_swin_route}. This result again indicates the excellent trade-off capability of the proposed high-order regularization method for the estimation bias, over-smoothness and numerical stability. The results in Table \ref{tab:route_unbias_hr_swin_results} reveal that the proposed HR method markedly improves localization accuracy, as evidenced by the total RMSE and the RMSE in the $z$ dimension, which shows enhancements of $67.4\%$ and $78.3\%$, respectively. These results are obtained under stringent conditions where the anchor heights are highly limited, with a maximum height differential of 0.5 meters in the simulations.

\subsection{Experiments}
In this section, we provide extensive experimental validations in the real-world environment to evaluate the performance of the proposed method. Our experimental platform is shown in Fig. \ref{Jackal_Robot}. Four anchors are used in these experiments to locate the robot by a tag carried on this robot. The coordinates of the anchors are measured manually using a rangefinder. The robot is designed to move along given trajectories with unknown dynamics. The position of the robot is calculated only by the current distance measurements between the robot and the four anchors, which means that no assumptions of the robot dynamics and the trajectory are required to calculate the position of the robot. An Ouster LiDAR is used to provide the reference positions of the robot, which are used to evaluate the performance of the proposed method. To reduce the influence of some unexpected factors, such as environment disturbance and UWB sensor antenna orientation, on distance measurements, a filter parameterized by the environment is applied to filter the distance measurements before calculating the position of the robot. To estimate the performance of the proposed method, the Umeyama method\cite{Umeyama1991Least} is used in these experiments to align two point patterns into the same frame referring to some evaluation tools\cite{grupp2017evo}.

\begin{figure}[!t]
\centering
\includegraphics[width=2.5in]{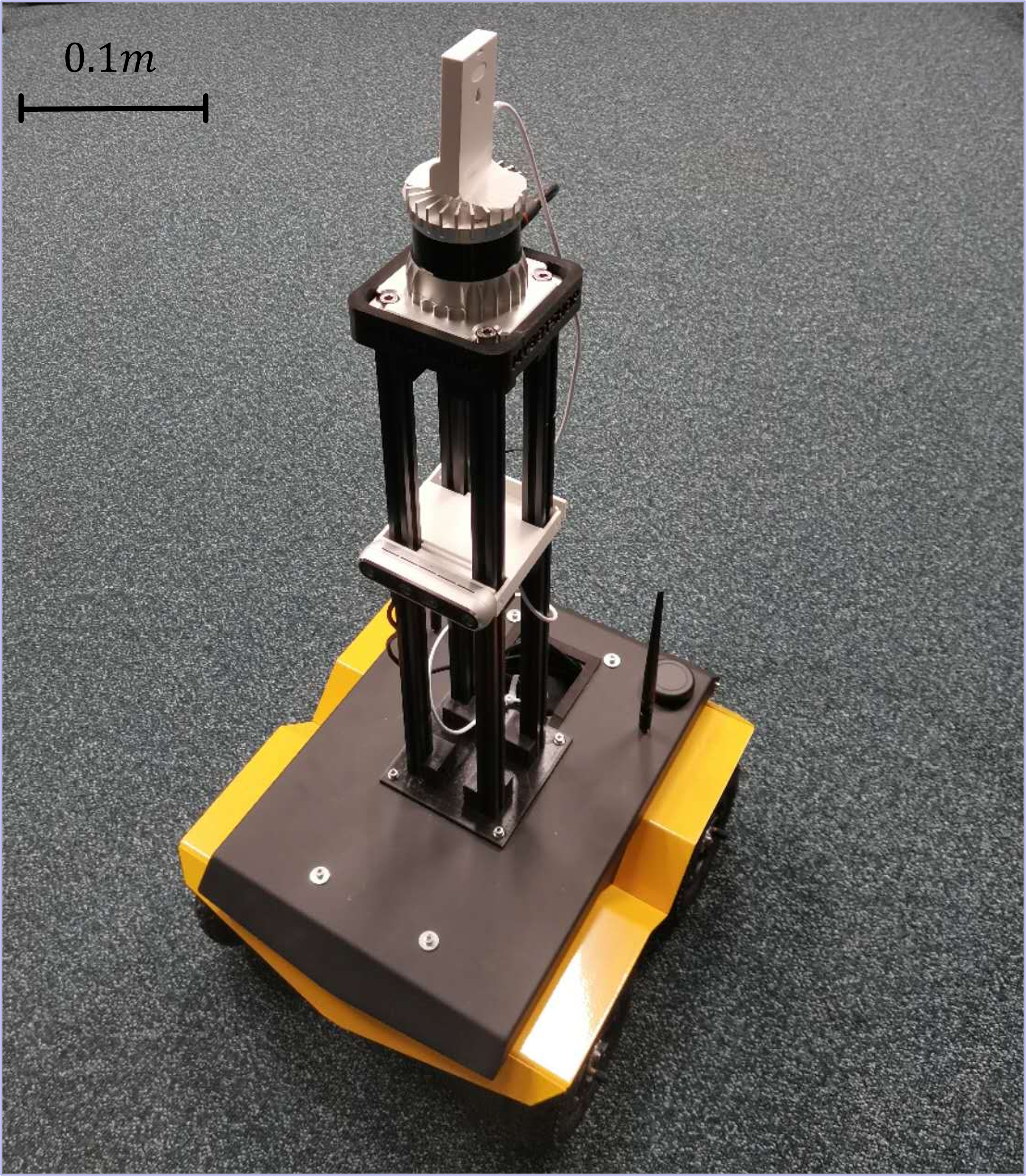}
\caption{The platform in our experiment. The white box on top is the UWB sensor that measures the distances between the robot and the anchors, and the LiDAR under the white box is used to provide the reference positions of the robot.}
\label{Jackal_Robot}
\end{figure}

\begin{figure}[!ht]
\centering
\includegraphics[width=3.5in]{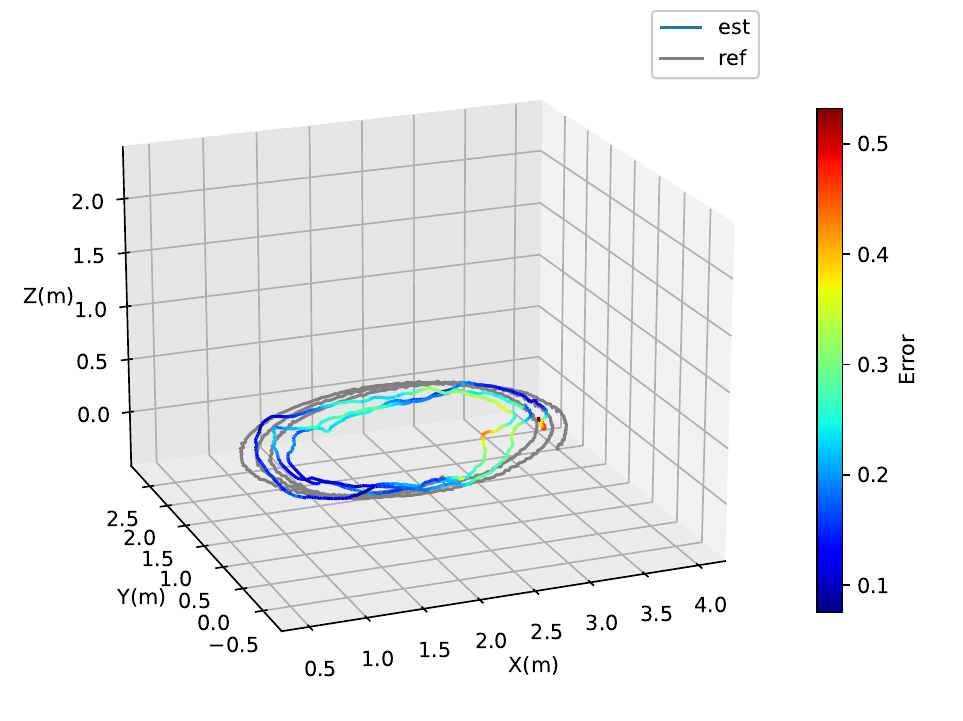}
\caption{The estimation result (est) and the reference path (ref) of the robot moving in a real 3D indoor environment. The reference path is calculated using the LiDAR measurements to compare the estimation of the proposed HR method.}
\label{real_route_result}
\end{figure}

\begin{figure}[!ht]
\centering
\includegraphics[width=3.5in]{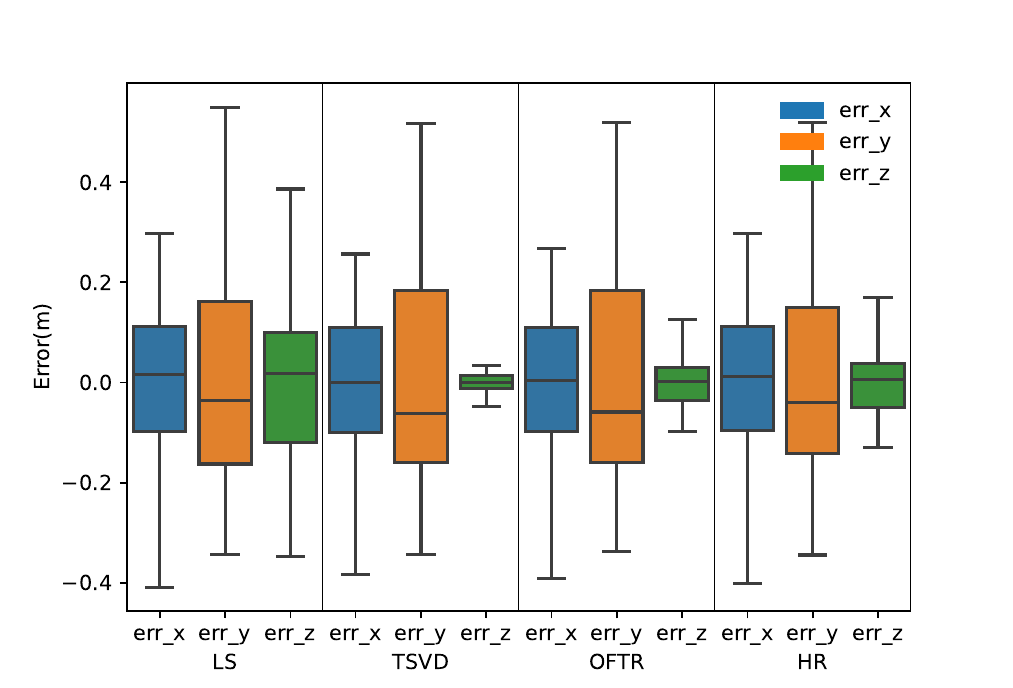}
\caption{The experiment results. The result is given by the high-order method with $R$ only changing the smallest eigenvalue of $A^TA$ and $\mu^2=(\lambda_n^2+\lambda_n \lambda_1)^{1/2} (s=n-1, k=1)$. The result of the OFTR solution is provided by the FTR method whose regularization parameter $\mu^2=(2\lambda_1 /\lambda_n)^{1/2}$ is given by the \emph{a priori} criterion proposed in this work for $s=n-1,k=0$, which only changes the smallest eigenvalue of the matrix $A^TA$. In the figure, the error in the z-direction of the proposed method is more similar to the errors in the $x$ and $y$ directions, compared with the other regularization methods.
}
\label{UWB_path_error_in_axces_box_024_all}
\end{figure}

\begin{figure}[!ht]
\centering
\includegraphics[width=3.5in]{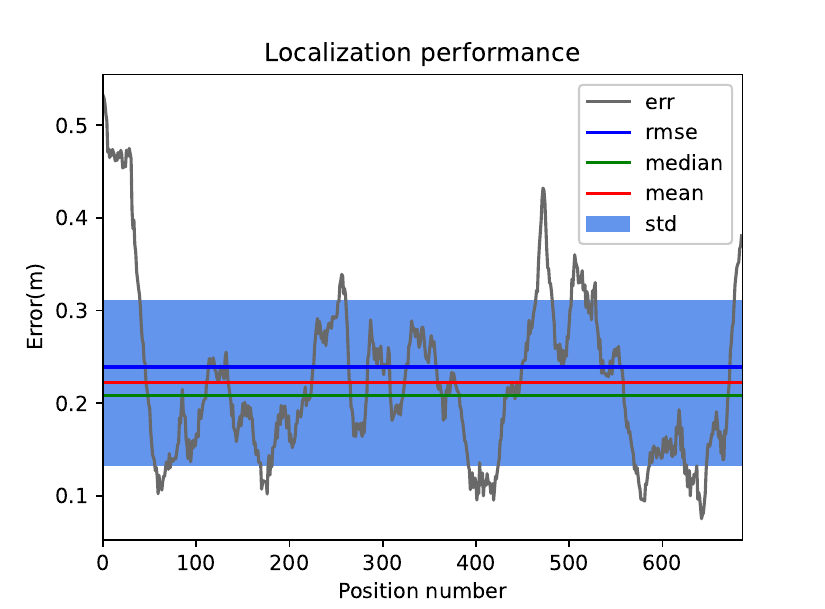}
\caption{The detailed performance of the proposed high-order method in a real 3D environment.}
\label{UWB_path_error_results_03021734_1120_201_HR}
\end{figure}

\begin{table}[!t]
\caption{Experimental results of different methods (unit: meter) \label{tab:Experimental_results}}
\centering
\begin{tabular}{c|c|c|c|c}
\hline
Method & RMSE  & Mean & Max & Min \\
\hline
LS & 0.28242 & 0.26636 & 0.60264 & 0.08551 \\
\hline
TSVD & 0.25220 & 0.23403 & 0.51938 & 0.06263 \\
\hline
OFTR & 0.25505 & 0.23779 & 0.51938 & 0.06203 \\
\hline
HR & \textbf{0.23915} & \textbf{0.22219} & 0.53203 & 0.07529 \\
\hline
\end{tabular}
\end{table}

\begin{table}[!t]
\caption{Experimental results of different methods with a different height of one anchor (unit: meter)\label{tab:Experimental_results_different_hight}}
\centering
\begin{tabular}{c|c|c|c|c}
\hline
Method & RMSE  & Mean & Max & Min \\
\hline
LS & 0.26989 & 0.25256 & 0.56059 & 0.06554 \\
\hline
TSVD & 0.24516 & 0.23229 & 0.48323 & 0.08341 \\
\hline
OFTR & 0.24557 & 0.23317 & 0.48225 & 0.08240 \\
\hline
HR & \textbf{0.21222} & \textbf{0.19865} & 0.45421 & 0.06630 \\
\hline
\end{tabular}
\end{table}

In the first two experiments, the robot moves on a slope, and the height of one anchor is changed in different experiments to reach different degrees of ill-conditioned situations. Three anchors are located at (0, 0, 2.29), (5.30, 4.12, 1.20), and (-0.56, 2.01, 0.30). The fourth anchor is either at (6.00, 0, 1.20) or higher at (6.00, 0, 1.80). The result and the reference of the localization experiment are shown in Fig. \ref{real_route_result}. The experimental results in Fig. \ref{UWB_path_error_in_axces_box_024_all} indicate that the proposed HR method improves the localization performance in the z-dimension, which is the same as the simulations. Comparing the results of TSVD, OFTR (one special case of HR with $k=0$ and an optimal $R$), and HR ($k=1$), we find that the error of TSVD or OFTR in z-dimension is smaller than HR, which is a severely over-smoothing situation, as the error of these methods in z-dimension is even much smaller than errors in x-dimension and y-dimension for the same methods. This situation occurs since our robot moves nearly in a plane. In some similar cases, we even obtain a smaller error in the z-dimension by assuming the change in z-dimension for the robot is 0, and we get the same small error result by using a larger regularization parameter $\mu$ for the regularization matrix $R$. However, the assumption does not hold for many applications (such as for drones) and is not required in the proposed high-order regularization. From Fig. \ref{UWB_path_error_in_axces_box_024_all}, the experimental results of the HR method in different dimensions are almost the same, which means that our high-order method overcomes the over-smoothing problem and ensures an improved localization performance. We also perform another experiment by changing the height of one anchor to show the behavior of the proposed HR method under different degrees of ill-conditioned situations. The detailed performance of the proposed high-order method is shown in Fig. \ref{UWB_path_error_results_03021734_1120_201_HR}, and all numerical results are given in Table \ref{tab:Experimental_results} and Table \ref{tab:Experimental_results_different_hight}. From the experimental results in Table \ref{tab:Experimental_results_different_hight}, the proposed HR method demonstrates a performance improvement in RMSE and mean error by over $20\%$ with the best enhancement of $21.3\%$, even when the anchor heights are not significantly limited, as in the experiments, the maximum height variation of anchors is 1.99 meters.

We also perform some tests on different trajectories. In these experiments, the robot moves horizontally in a $7m\times7m$ room. The UWB anchors are located at (1.50, 2.00, 0.69), (1.50, 0.17, 2.00), (7.17, 2.00, 1.50), and (7.20, 0.50, 3.00), where the anchors exhibit very close spacing along the y and z axes compared to the x axis.
% the UWB anchors are spaced along a much closer space both in the y and z than in x. 
In the `shape 1' experiment, the robot drove in a straight line, almost orthogonal to the x direction of the anchor layout. In the `shape 8' experiment, the robot moves along a figure-8 trajectory. For the random experiment, the robot drove randomly in the room. The results are given in Table \ref{tab:Experimental_results_differnt_routes}. The results of these experiments reaffirm the superiority of the proposed high-order regularization method over the other methods in terms of RMSE and mean error in such a severely restricted environment.
\begin{table}[!t]
\caption{Results of different routes (unit: meter) \label{tab:Experimental_results_differnt_routes}}
\centering
\hspace{-0.51cm}
% \vspace{0cm}
\begin{tabular}{c|cc|cc|cc}
\hline
\multirow{2}{*}{Routes} & \multicolumn{2}{c|}{shape 1}     & \multicolumn{2}{c|}{shape 8 }     & \multicolumn{2}{c}{random}     \\ \cline{2-7} 
                   & \multicolumn{1}{c|}{RMSE} & Mean & \multicolumn{1}{c|}{RMSE} & Mean & \multicolumn{1}{c|}{RMSE} & Mean \\ \hline
LS     & \multicolumn{1}{c|}{1.31241} &1.03008 & \multicolumn{1}{c|}{1.25969} &0.99082 & \multicolumn{1}{c|}{1.43326} &1.04008 \\ \hline
TSVD   & \multicolumn{1}{c|}{0.34737} &0.28921 & \multicolumn{1}{c|}{0.30466} &0.24939 & \multicolumn{1}{c|}{0.33180} &0.28154 \\ \hline
OFTR   & \multicolumn{1}{c|}{0.34746} &0.28936 & \multicolumn{1}{c|}{0.30474} &0.24953 & \multicolumn{1}{c|}{0.33191} &0.28183 \\ \hline
HR     & \multicolumn{1}{c|}{\textbf{0.25541}} &\textbf{0.21260} & \multicolumn{1}{c|}{\textbf{0.25968}} &\textbf{0.22579} & \multicolumn{1}{c|}{\textbf{0.25475}} &\textbf{0.23577} \\ \hline
\end{tabular}
\end{table}

\section{Conclusion and discussion}\label{conclusions}
In this work, we propose a high-order regularization method to solve the restricted robot localization problem when the associated matrix is ill-conditioned. From the simulation, in a severely restricted scenario, the proposed HR method achieves a significant enhancement of $78.3\%$ in $z$ dimension (the restricted dimension), while in a scenario without severe restrictions, the improvement of the proposed method also reaches more than $20\%$ from the experimental results. We show that the proposed method is superior to many Tikhonov regularization-based methods in approximating the inverse problem. It is shown that the Tikhonov regularization is a special case of the proposed method. The explanation of the over-smoothing of the Tikhonov regularization solution is also given. The proposed method is based on matrix series expansion and eigenvalue decomposition, which means that there are always some explainable and reasonable techniques to improve the estimation performance of the proposed method as eigenvalue decomposition is always applicable to the new matrix.

To find an explainable and reasonable regularization matrix, we simplify its selection by adopting a special form, which only changes some small eigenvalues of the matrix $A^T A$. We also propose one \emph{a priori} criterion to improve the numerical stability of the problem, which is regarded as replacing the ill-conditioned problem with a well-conditioned approximate problem. The optimal regularization matrix is calculated based on the proposed criterion, which considers the balance between modifying small eigenvalues and keeping the information of some dimensions with respect to these small eigenvalues. However, if the information on some dimensions is not needed or is obtained by some other technologies, the proposed criterion does not guarantee the best choices for the regularization matrix. The optimal regularization matrix also provides the potential to verify if the given $k$ or $s$ is appropriate. Two methods to improve the performance of the $k$th-order regularization and calculate the bias of the regularization solution are also proposed, which helps achieve a possible unbiased solution similar to the LS solution.
Note that the computation of $R$, in the case of the optimal regularization matrix according to \eqref{Relaxation_R}, involves the computation of eigenvalues of matrices, which causes computational tractability when the dimension of the matrix grows. In such scenarios, our proposed approach for selecting $R$ applies only to smaller scale problems. Please also note that one does not always have to pick the optimal $R$; as long as the spectral radius inequality below equation \eqref{HR_soluntion} is satisfied, to improve computational efficiency, one can choose $R$ such that $A^TA +R$ becomes diagonal, block diagonal or triangular.

For future work, some possible improvements of the high-order regularization method are determining the adjustment parameter by calculating the residual directly or calculating the bias with some expressions with respect to the new matrix and the regularization matrix. Extending such a high-order regularization to nonlinear model cases is the next step of this work. Applying the high-order regularization to some other ill-conditioned problems is also a promising topic. In particular, applying the proposed explainable and bounded high-order regularization to machine learning, which uses regularization methods widely to achieve generalization abilities, is expected to drive a more explainable machine learning theory.

\section*{Acknowledgments}
The authors would like to thank Bayu Jayawardhana for the UWB sensors and Simon Busman for his help in the experiments.

{\appendices

\section{The problem that the HR solves}
Similar to TR and LS, one can derive the problem corresponding to the proposed high-order regularization method for $k>0$. Consider that the high-order solution is achieved by setting the first derivative of some cost function to zero, and one can assume that the first derivative of the problem $f_x$ has

\begin{equation} \label{HR_problem_derivative}
\begin{aligned}
& 0=\left(\sum_{i=0}^k\left(R\left(A^T A+R\right)^{-1}\right)^i\right)^{-1} \left(A^T A+R\right)^{-1}x  -A^T b,
\end{aligned}
\end{equation}
which is rewritten as
\begin{equation}
\begin{aligned}
& 0=A^T A x-A^T b+R x-\\ &\left(\sum_{i=0}^k\left(R\left(A^T A+R\right)^{-1}\right)^i\right)^{-1} R \sum_{i=0}^{k-1}\left(R\left(A^T A+R\right)^{-1}\right)^i x, \\
& :=\frac{1}{2}\frac{\partial}{\partial x}\left[(A x-b)^{\mathrm{T}}(A x-b)\right]+\frac{1}{2}\frac{\partial}{\partial x}\left[x^T B(R)^T B(R) x\right],\\
% &=\frac{\partial f_x}{\partial x}
\end{aligned}
\end{equation}
where 
\begin{equation}
\begin{aligned}
& B(R)^T B(R)=R-\\
&\left(\sum_{i=0}^k\left(R\left(A^T A+R\right)^{-1}\right)^{i}\right)^{-1} R\sum_{i=0}^{k-1}\left(R\left(A^T A+R\right)^{-1}\right)^i.
\end{aligned}
\end{equation}
Hence, the problem solved by the HR solution is
\begin{equation}
\begin{aligned}
\min \{\|A x-b\|^2+\|B(R) x\|^2\}.
\end{aligned}
\end{equation}
It follows that 

\begin{equation}
\lim _{k \rightarrow \infty} B(R) \rightarrow O_n,
\end{equation}
which implies that the problem is an approximate problem to the original problem. Note that for the HR solution, there are many equivalent problems due to different selections of the first derivative expression. This result shows that the HR solution is regarded as an approximation solution to the original problem.
% However, as most of these problems provide limited extra explanations of the solution, we prefer to consider the HR solution to be an approximation solution to the original problem rather than referring to it as the solution to a new problem.

\section{Proof of a similar approximate expression}
Consider the Cholesky factorization of $A^T A$ is $S^T S$, then one has
% as follows
% \begin{equation}
% A^T A=S^T S.
% \end{equation}
% And then one has
\begin{equation}
\left(A^T A+R\right)=\left(S^T S+R\right).
\end{equation}
If the spectral radius condition $\rho\left(\left(S^T\right)^{-1} RS^{-1}\right)<1$ is satisfied, then a special and similar approximate expression of the inverse of $A^T A$ is found as follows
\begin{equation}
\begin{aligned}
& \left(A^T A+R\right)^{-1} \\
& =S^{-1}\left[I+\left(S^T\right)^{-1} RS^{-1}\right]^{-1}\left(S^T\right)^{-1} \\
% & =(S)^{-1}\left[I-\left(-\left(S^T\right)^{-1} R(S)^{-1}\right)\right]^{-1}\left(S^T\right)^{-1} \\
% & =(S)^{-1}\left[I+\sum_{i=1}^{\infty}\left(-\left(S^T\right)^{-1} R(S)^{-1}\right)^i\right]\left(S^T\right)^{-1} \\
% & =\left(A^T A\right)^{-1}+(S)^{-1}\left[\sum_{i=1}^{\infty}\left(-\left(S^T\right)^{-1} R(S)^{-1}\right)^i\right]^{-1}\left(S^T\right)^{-1} \\
% & =\left(A^T A\right)^{-1}+\left(A^T A\right)^{-1} \sum_{i=1}^{\infty}\left(-R(S)^{-1}\left(S^T\right)^{-1}\right)^i \\
% & =\left(A^T A\right)^{-1}+\left(A^T A\right)^{-1} \sum_{i=1}^{\infty}\left(-R\left(A^T A\right)^{-1}\right)^i \\
% & =\left(A^T A\right)^{-1} \sum_{i=0}^{\infty}\left(-R\left(A^T A\right)^{-1}\right)^i \\
% & =\left(A^T A\right)^{-1}-\left(A^T A\right)^{-1} R\left(A^T A\right)^{-1} \\
% & \quad+\left(A^T A\right)^{-1} \sum_{i=2}^{\infty}\left(-R\left(A^T A\right)^{-1}\right)^i
& = \left(A^T A\right)^{-1} \sum_{i=0}^{\infty}\left(-R\left(A^T A\right)^{-1}\right)^i,
\end{aligned}
\end{equation}
and the spectral radius becomes $\rho\left(-R\left(A^T A\right)^{-1}\right)<1$. Then one has
\begin{equation} \label{AA_AAR}
\begin{aligned}
\left(A^T A\right)^{-1}=&\left(A^T A+R\right)^{-1}+\left(A^T A\right)^{-1} R\left(A^T A\right)^{-1}\\
&-\left(A^T A\right)^{-1} \sum_{i=2}^{\infty}\left(-R\left(A^T A\right)^{-1}\right)^i,
\end{aligned}
\end{equation}
and
\begin{equation}\label{RAA_RAAR}
\begin{aligned}
& -R\left(A^T A\right)^{-1}+R\left(A^T A+R\right)^{-1} \\
% & =-R\left(\left(A^T A\right)^{-1}-\left(A^T A+R\right)^{-1}\right) \\
% & =-R\Bigg(\left(A^T A\right)^{-1} R\left(A^T A\right)^{-1} \\ &\quad \quad\quad\quad-\left(A^T A\right)^{-1} \sum_{i=2}^{\infty}\left(-R\left(A^T A\right)^{-1}\right)^i\Bigg)\\
&=-R\left(A^T A\right)^{-1} \left(R\left(A^T A\right)^{-1} -\sum_{i=2}^{\infty}\left(-R\left(A^T A\right)^{-1}\right)^i\right).
\end{aligned}
\end{equation}
When $k=1$, the bias between LS solution \eqref{Ls_soluntion} and the high-order solution \eqref{HR_soluntion_k1} is
\begin{equation}
\begin{aligned}
\hat{x}_{ls}-\hat{x}_{hr}=&\left(A^T A\right)^{-1}A^T b-\left(A^T A+R\right)^{-1}A^T b\\
&-\left(A^T A+R\right)^{-1} R\left(A^T A+R\right)^{-1} A^T b,
\end{aligned}
\end{equation}
Note that, one has
\begin{equation}
\begin{aligned}
& \left(A^T A\right)^{-1}-\left(A^T A+R\right)^{-1}-\left(A^T A+R\right)^{-1} R\left(A^T A+R\right)^{-1}\\
& =\left(\left(A^T A\right)^{-1}-\left(A^T A+R\right)^{-1}\right) R\left(A^T A\right)^{-1}\\
& \quad +\left(A^T A+R\right)^{-1}\left(R\left(A^T A\right)^{-1}-R\left(A^T A+R\right)^{-1}\right) \\
& \quad -\left(A^T A\right)^{-1} \sum_{i=2}^{\infty}\left(-R\left(A^T A\right)^{-1}\right)^i,
\end{aligned}
\end{equation}
and then according to \eqref{AA_AAR} and \eqref{RAA_RAAR}, one has the approximation error of the matrix inverse
\begin{equation}
\begin{aligned}
& err_{ap}((A^TA)^{-1})=\\
& \left(A^T A\right)^{-1}-\left(A^T A+R\right)^{-1}-\left(A^T A+R\right)^{-1} R\left(A^T A+R\right)^{-1}\\
& =\left(A^T A+R\right)^{-1}\left(\left(I+R\left(A^T A\right)^{-1}\right)^{-1}-I+R\left(A^T A\right)^{-1}\right).
\end{aligned}
\end{equation}

Hence, the bias is described as
\begin{equation}
\begin{aligned}
& \hat{x}_{ls}-\hat{x}_{hr} \\
&=\left(A^T A+R\right)^{-1}\\
& \quad \times \left(\left(I+R\left(A^T A\right)^{-1}\right)^{-1}-I+R\left(A^T A\right)^{-1}\right) A^T b.
\end{aligned}
\end{equation}
This result shows that for $k=1$, the bias between the LS solution and the high-order solution can be described in different forms under some different approximate expressions of the inverse of $A^TA$. However, some approximate expressions do not guarantee a stable solution to the ill-conditioned problem, as in these approximate expressions, the inverse $(A^TA)^{-1}$ still needs to be calculated or is only suitable for some $k$.

\section{Comparison of error bounds of the HR solution and the TR solution}
Note that $\rho\left(R\left(A^T A+R\right)^{-1}\right)<1$, and one has the approximate residual \eqref{Approximation_residual_F}. For $k>0$, one has the error bounds of the HR solution $\left\|\hat{x}_{hr}-\hat{x}_{ls}\right\|$ and the TR solution
\begin{equation}
\begin{aligned}
&\left\|\hat{x}_{tr}-\hat{x}_{ls}\right\| \\
& = \left\|\left(A^T A+R\right)^{-1} \left(\sum_{i=1}^{k}\left(R\left(A^T A+R\right)^{-1}\right)^i\right)A^T b  \right.\\
& \quad\quad \left. + \hat{x}_{hr}-\hat{x}_{ls} \right\| \\
& \geq \left\|\hat{x}_{hr}-\hat{x}_{ls}\right\|.
\end{aligned}
\end{equation}
This result is also obtained from \eqref{Bias_LS_HR} since TR is a special case of HR.

\section{Proof of convexity of the loss function}
The following proof shows that the loss function is a convex function, for $s=n-1, k \in \mathbb{N}$. For $s=n-1, k=0$ and $k=1$, we have shown that the loss function in \eqref{Loss_function_relaxation} is a convex function. Note that if $\mu^2=\lambda_n$, $R=O_n$ . Assume $s=n-1, k \geq 2$ and $\mu^2 > \lambda_n$, one has
\begin{equation}
\begin{aligned}
g(\mu):=\frac{\mu}{\lambda_n}\left(\frac{\mu-\lambda_n}{\mu}\right)^{k+1}=\frac{1}{\lambda_n} \mu\left(1-\frac{\lambda_n}{\mu}\right)^{k+1},
\end{aligned}
\end{equation}

\begin{equation}
\begin{aligned}
&g^{\prime}(\mu) =\left(\frac{\mu+k \lambda_n}{\left(\mu-\lambda_n\right) \mu}\right) g(\mu),
\end{aligned}
\end{equation}

\begin{equation}
\begin{aligned}
& g ^{\prime \prime}(\mu) =\left(\frac{\left(k+k^2\right) \lambda_n^2}{\left(\left(\mu-\lambda_n\right) \mu\right)^2}\right) g(\mu) \\ &=\frac{\lambda_n\left(k+k^2\right)}{\mu^2\left(\mu-\lambda_n\right)}\left(\frac{\mu-\lambda_n}{\mu}\right)^k>0.
\end{aligned}
\end{equation}
Then the second derivative of the loss function in \eqref{Loss_function_relaxation} is expressed as
\begin{equation}
\begin{aligned}
& f^{\prime \prime}(\mu^2) = g^{\prime \prime}(\mu^2)+\frac{2 \lambda_1}{\mu^6}\\
& =\frac{\lambda_n\left(k+k^2\right)}{\mu^4\left(\mu^2-\lambda_n\right)}\left(\frac{\mu^2-\lambda_n}{\mu^2}\right)^k+\frac{2 \lambda_1}{\mu^6}>0.
\end{aligned}
\end{equation}
Hence, for $s=n-1, k \in \mathbb{N}$, the loss function is a convex function with respect to $\mu^2$. Therefore, there must be a unique minimum of the loss function, and an optimal regularization matrix $R$ is found at the minimum of the loss function.

\section{The covariance calculation of HR}
Considering that $b_0=b+\epsilon$ is the observation with noise $\epsilon$, one has 
\begin{equation}
\begin{aligned}
&\hat{x}_{hr} - x = \Theta A^T \epsilon,
\end{aligned}
\end{equation}
where $\Theta:= \left(A^T A+R\right)^{-1} \sum_{i=0}^k\left(R\left(A^T A+R\right)^{-1}\right)^i$ is the matrix inverse approximation of $(A^T A)^{-1}$. Hence, similar to LS, one can calculate a covariance by
\begin{equation}
\begin{aligned}
&C_{\hat{x}_{hr}}=\text{E}[(\hat{x}_{hr} - x)(\hat{x}_{hr} - x)^T] = \Theta A^T \text{E}[\epsilon \epsilon ^T] A \Theta ^T ,
\end{aligned}
\end{equation}
where $\text{E}[\cdot]$ is the expectation operator. For $\text{E}[\epsilon \epsilon ^T]=\sigma^2 I$, the covariance matrix $C_{\hat{x}_{hr}}= \sigma^2 \Theta A^T A \Theta ^T$, and one has 
\begin{equation}\lim _{k \rightarrow \infty} C_{\hat{x}_{hr,k}}= \sigma^2 (A^T A)^{-1},
\end{equation}
which is the covariance of the LS estimator.

}

% \section{References Section}

% \bibliographystyle{IEEEtran}
% \bibliography{references}

\begin{IEEEbiography}[{\includegraphics[width=1in,height=1.25in,clip,keepaspectratio]{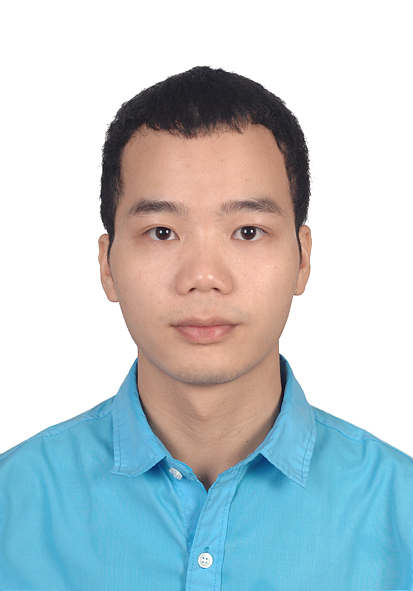}}]{Xinghua Liu}
(Student Member, IEEE) received the bachelor’s degree in Measurement, Control Technology and Instrument from Hefei University of Technology, Anhui, China, in 2017, and the master’s degree in Instrument Science and Technology from Southeast University, Jiangsu, China, in 2020. He is currently working toward the Ph.D. degree with the University of Groningen, Groningen, the Netherlands.

He received the Outstanding Graduate from Hefei University of Technology in 2017. He was a research intern in 2018 with Nanjing Fujitsu Nanda Software Technology Co., Ltd., China. From 2020 to 2022, he was an R$\&$D engineer with Huawei Technologies Co. Ltd., China, contributing to AI computing networks. His research interests include robotics, artificial intelligence, applied physics, multi-sensor fusion, and learning-based control.
\end{IEEEbiography}

\vspace{11pt}

\begin{IEEEbiography}[{\includegraphics[width=1in,height=1.25in,clip,keepaspectratio]{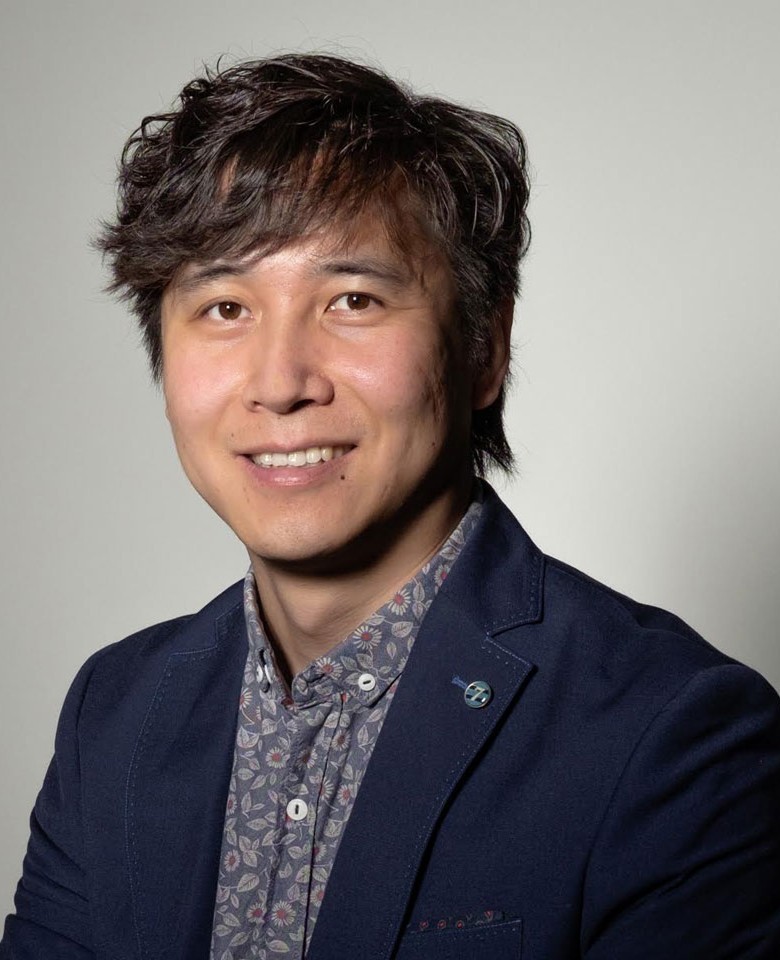}}]{Ming Cao}
(Fellow, IEEE) received the bachelor’s and master’s degrees from Tsinghua University, Beijing, China, in 1999 and 2002, respectively, and the Ph.D. degree from Yale University, New Haven, CT, USA, in 2007, all in electrical engineering.

Since 2016, he has been a professor of Networks and Robotics with the Engineering and Technology Institute, University of Groningen, Groningen, The Netherlands, where he started as an Assistant Professor in 2008. From 2007 to 2008, he was a Research Associate with Princeton University, Princeton, NJ, USA. He was a Research Intern in 2006 with the IBM T. J. Watson Research Center, USA. 

He is an IEEE Fellow. He is the 2017 and inaugural recipient of the Manfred Thoma medal from the International Federation of Automatic Control (IFAC) and the 2016 recipient of the European Control Award sponsored by the European Control Association (EUCA). He is a Senior Editor for Systems and Control Letters, and is or has been an Associate Editor for IEEE Transactions on Automatic Control, IEEE Transaction on Control of Network Systems, IEEE Transactions on Neural Networks and Learning Systems, IEEE Transactions on Circuits and Systems, IEEE Robotics $\&$ Automation Magazine, and IEEE Circuits and Systems Magazine. He is a member of the IFAC Council and a vice chair of the IFAC Technical Committee on Large-Scale Complex Systems. His research interests include autonomous robots and multi-agent systems, complex networks and decision-making processes.
\end{IEEEbiography}

\vfill

\end{document}